\documentclass[twoside,11pt]{jmlr} 

\usepackage{times,mdframed,todonotes,wrapfig}
\usepackage{graphicx,graphics}

\usepackage{amsmath}
\usepackage{amssymb}
\usepackage{xspace}
\usepackage{bm}
\usepackage{algorithm}
\usepackage{bbm}
\usepackage[most]{tcolorbox}

\newlength{\minipagewidth}
\usepackage{times}
\usepackage{hyperref}
\usepackage[english]{babel}
\usepackage{amscd,ifthen}
\usepackage{xcolor}
\usepackage{mdframed}
\usepackage{dsfont}
\usepackage{framed}
\usepackage{multirow}
\usepackage{transparent}
\usepackage{tikz}
\usepackage{color}
\usepackage{dirtytalk}
\usepackage{mathtools}
\usepackage{framed}
\usepackage{array}
\usepackage{soul}
\usepackage[margin=1cm]{caption}
\usepackage[noend]{algpseudocode}

\makeatletter
\def\BState{\State\hskip-\ALG@thistlm}
\makeatother

\newcommand{\eps}{\epsilon}
\newcommand{\E}{\mathbb{E}}

\newcommand{\Ind}{\mathds{1}} 
\newcommand{\F}{\mathcal{F}}
\newcommand{\G}{\mathcal{G}}

\newcommand{\X}{\mathcal{X}}

\DeclareMathOperator*{\argmin}{arg\,min}
\DeclareMathOperator*{\polylog}{polylog}
\DeclareMathOperator*{\dis}{DIS}

\newcommand{\D}{\mathcal{D}}

\renewcommand{\eps}{\varepsilon}

\definecolor{cd60952}{RGB}{214,9,82}

\renewcommand{\leq}{\leqslant}
\renewcommand{\geq}{\geqslant}
\renewcommand{\le}{\leqslant}
\renewcommand{\ge}{\geqslant}

\newenvironment{myproof}[1]%
{%
   \par\noindent{\bfseries\upshape Proof #1\ }%
}%
{\jmlrQED}

\date{}

\newtheorem{Theorem}{Theorem}[section]
\newtheorem{Lemma}[Theorem]{Lemma}

\newtheorem{Proposition}[Theorem]{Proposition}

\newtheorem{Remark}[Theorem]{Remark}
\newtheorem{Example}[Theorem]{Example}
\newtheorem{myalgorithm}[Theorem]{Algorithm}

\title[Exponential Savings in Agnostic Active Learning through Abstention]{Exponential Savings in Agnostic Active Learning through Abstention}%

\author{\Name[{Nikita~Puchkin}]{Nikita Puchkin} \Email{npuchkin@hse.ru}\\
 \addr HSE University and Institute for Information Transmission Problems RAS, 
	Moscow
 \\
 \Name[{Nikita~Zhivotovskiy}]{Nikita Zhivotovskiy} \Email{nikita.zhivotovskii@math.ethz.ch}\\
 \addr ETH, Z\"urich
 }

\begin{document}

\maketitle

\begin{abstract}
We show that in pool-based active classification without assumptions on the 
underlying distribution, if the learner is given the power to abstain from some predictions 
by paying the price marginally smaller than the average loss $1/2$ of a random guess, 
exponential savings in the number of label requests are possible whenever they are 
possible in the corresponding realizable problem. We extend this result to provide a 
necessary and sufficient condition for exponential savings in pool-based active 
classification under the model misspecification.
\end{abstract}


\section{Introduction}
\label{sec:introduction}
Pool-based \emph{active classification} can be seen as an extension of 
the classical PAC classification setup, where instead of learning from the labeled sample 
$(X_1, Y_1), \ldots, 
(X_n, Y_n)$, one can adaptively request the labels from a large pool $X_{1}, X_2, 
\ldots$ of i.i.d. unlabeled instances round by round.
Our hope is to request significantly fewer labels $Y_i$ and get the same statistical 
guarantees as in \emph{passive learning}. A textbook example is the one of learning the 
class $\F$ of threshold 
classifiers in the realizable (noise-free) case where a binary search based algorithm can 
improve 
the sample complexity (the number of requested labels) from the passive
sample complexity $O(\frac{1}{\varepsilon})$ to the exponentially better $O(\log 
\frac{1}{\varepsilon})$ sample complexity, where 
$\varepsilon$ is the desired probability of error. For more general classes the realizable 
case sample complexity is understood quite well in the distribution dependent 
\citep{dasgupta2005coarse} and the minimax \citep{hanneke2015minimax} senses.

The improvements in active learning are less impressive once the problem is not 
realizable. Our starting point is the foundational work of \cite{Kaariainen05} containing
the following observations formulated (informally) as follows:
\begin{enumerate}
	\item (\emph{Arbitrary noise}) Active learning cannot bypass the classical agnostic 
	passive learning bound on the sample complexity $\Omega(\frac{1}{\varepsilon^2})$ 
	\citep{Vapnik74} in the noisy case\footnote{When only the dependence on $\eps$ is 
	considered.}. 
	Indeed, if there is just one \say{heavy} instance $X$ with the noisy label $Y \in \{0, 
	1\}$ such that $\Pr(Y = 1|X) = \frac{1}{2}\pm \varepsilon$, the sample complexity of 
	any active learning 
	algorithm can be reduced to the estimation of $\Pr(Y = 1|X)$. A well-known result 
	\citep[Lemma 5.1]{anthony2009neural} gives the lower bound $\Omega(\frac{1}{\varepsilon^2})$.
	We give a more detailed discussion in Example \ref{ex:large_noise}.  
	\item (\emph{Misspecification}) No improvements in active learning over passive 
	learning are 
	possible in the misspecified case even if there is no noise in the labeling mechanism. 
	That is, even when $Y = f^*(X)$ 
	almost surely for some $f^* \notin \F$, one cannot bypass the passive learning lower 
	bound $\Omega(\frac{1}{\varepsilon^2})$. To demonstrate this phenomenon a class 
	$\F$ consisting of only two specific functions is sufficient.
	\item (\emph{Bounded noise}) The sample complexity of active classification in the 
	bounded noise case, that is, when each label of the true function is corrupted 
	independently with probability strictly smaller than $0.5$, is essentially the same as 
	in the realizable case whenever the Bayes optimal rule is in the class. Therefore, in 
	this case, exponential savings are also possible and well understood by now.
\end{enumerate}

To avoid the aforementioned lower bounds of \citeauthor{Kaariainen05} and to show 
significant superiority of active learning, many authors are focusing on the favorable 
\emph{noise assumptions}. These assumptions take their roots in passive learning and 
include the realizable or the bounded noise cases \citep{Massart06}, Tsybakov's noise 
\citep{Tsybakov04}, and the Bernstein condition 
\citep{Bartlett06}. Unfortunately, these assumptions are difficult to satisfy in practice, as 
they require that the Bayes optimal classifier is in the class. The Bernstein condition 
avoids this problem, but little is known about the cases where it holds in classification 
without 
assuming that the Bayes optimal rule is in the class (see \citep{Elyaniv17}). Based on his findings, \citeauthor{Kaariainen05} writes: \say{\emph{The implication of this lower bound is that exponential savings\footnote{Throughout the paper, by exponential savings, we mean exponential improvements with respect to the dependence on $\eps$ only.} should not be expected in realistic models of active learning, and thus the label complexity goals in active learning should be refined}}. 

This paper aims to provide such a refinement. Our 
method will be as follows: instead of restricting the distribution of $(X, Y)$, we use the 
power to abstain from some predictions. To do so, we allow the learner to output a $\{0, 
1, *\}$-valued classifier, where $*$ corresponds to the reject option. For $p \in [0, 
\frac{1}{2}]$ and a $\{0, 1, *\}$-valued classifier $f$ we define the Chow's risk 
\citep{Chow70} as
\begin{align}
\label{rp}
    R^p(f)
    &\notag
    = \Pr\left(f(X) \neq Y\ \text{and}\ f(X) \in \{0, 1\}\right)
    \\&\quad
    + \left(\frac{1}{2} - p\right)\Pr(f(X) = *),
\end{align}    
which is the binary risk as long as we predict in $\{0, 1\}$, and the price of abstention is 
equal to $\frac{1}{2} - p$. The special case $p = 0$ corresponds to the 
situation where the price of abstention is the same as the average loss of a random 
guess. In what follows, we always think of $p$ as a small parameter, so that the price of 
$\frac{1}{2} - p = 0.49$ for abstention will always suffice.

In the standard active learning setup, given a class $\F$ of $\{0, 1\}$-valued classifiers 
and using only a small number of label requests, we aim to construct a classifier 
$\widetilde{f}$ such that, with high probability, 
\[
R(\widetilde{f}) - \inf\limits_{f \in \F}R(f) \le \varepsilon,
\]
where the binary risk $R(f)$ is defined as $R(f) = \Pr(f(X) \neq Y)$.
Our simplified aim is instead to construct for the same class $\F$ a $\{0, 1, *\}$-classifier 
$\widehat{f}$ such that, with high probability, 
\begin{equation}
\label{eq:rpexcessrisk}
R^p(\widehat{f}) - \inf\limits_{f \in \F}R(f) \le \varepsilon,
\end{equation}
that is, together with a reject option, $\widehat{f}$ predicts almost as good as the 
best classifier in the class. This approach has a practical motivation: if the learner can 
identify the instances where the output classifier predicts no better than a random 
guess, it is reasonable to use some external source of information (for example, expert 
advice) to make a prediction. It will be clear that our setup is essentially an active 
learning version of the \emph{model selection aggregation} problem \citep{Tsybakov03}; 
in the model selection aggregation one is allowed to output an \emph{improper} (not 
necessarily in $\F$) classifier and use the \say{curvature} of the loss function to predict 
as good as the best classifier in $\F$. 
It has been recently shown in \citep{bousquet2019fast} that Chow's risk \eqref{rp} gives 
exactly the right amount of \say{curvature} to exploit the techniques in the model 
selection aggregation and improve the agnostic sample complexity in passive learning. 
Another closely related setup is the prediction of individual sequences with expert 
advice where the \say{curvature} of the loss expressed in terms of 
\emph{mixability} gives significant improvements with no 
assumptions on the data generating mechanism \citep{Vovk1990, 
	haussler1998sequential} (see also \citep[Chapter 3]{cesa2006prediction}).

\begin{Example}
    \label{ex:large_noise}
	Fix $\eps > 0$.
	Let the instance space consist of only one instance $x_0$ and assume that $Y$ is such that
	\[
		\Pr(Y = 1 \,\vert\, X = x_0) = \frac12 + \sigma \eps,
	\]
	where $\sigma \in \{-1, 1\}$.
	Consider a class $\F = \{f_0, f_1\}$ where $f_0(x_0) = 0$, $f_1(x_0) = 1$.
	One of these two classifiers has a risk $1/2 - \eps$, and the risk of the other one is $1/2 + \eps$ (independently of the sign of $\sigma$).
	Thus, if a learner is allowed to produce only $\{0, 1\}$-valued classifier, they must determine the sign of $\sigma$ exactly.
	If they fail, the excess risk of the estimator will be $2\eps > \eps$.
	According to \citep[Lemma 5.1]{anthony2009neural}, any active learning algorithm requires $\Omega(1/\eps^2)$ labels to find the best classifier among $f_0$ and $f_1$.
	However, if we consider a classifier $\widehat f$ with the reject option such that $\widehat f(x_0) = *$, then
	\begin{align*}
		R^p(\widehat f) - \min\limits_{f \in \F} R(f)
		&
		= \left(\frac12 - p\right) - \left(\frac12 - \eps\right)
		\\&
		= \eps - p < \eps, \quad \text{for all $p>0$.}
	\end{align*}
	Hence, using the reject option, we constructed a rule with the excess risk smaller than $\eps$ and have not requested any labels.
	This simple example illustrates how the reject option helps to reduce label requests on noisy instances in some cases.
\end{Example}

In what follows, we assume that the price of abstention is only 
marginally smaller than the average loss of a random guess\footnote{By random guessing, we mean that there is an external randomization mechanism which replaces $*$ either by $0$ or $1$ with equal probabilities. The risk is then computed as an average with respect to this randomization.}; that is, for instance, 
$\frac{1}{2} - p = 0.49$. With such an option, we have to be conservative when to 
abstain: 
our algorithm should be adaptive to the bounded noise case 
and avoid the reject option in this situation. Indeed, if the distribution is such that the 
true labels are corrupted independently with a probability of only $0.25$, any abstention 
with the price of $0.49$ can worsen the situation. We show in Proposition 
\ref{prop:adaptivity} that our algorithm is adaptive to the bounded noise assumption and 
abstains rarely if this assumption holds; 
thus, we are always recovering the standard guarantees.

We are ready to make an informal statement of our main result. We use the notions 
of the \emph{VC dimension} and the \emph{disagreement coefficient}, both of which are 
standard in the active learning literature. Formal definitions are presented in Section 
\ref{sec:notation}.
\begin{Theorem}[A simplified statement]
	Fix $\eps, \delta \in (0, 1]$ and $p \in (0, \frac{1}{2}]$. There is an active learning 
	algorithm such that for \emph{any} distribution $P$ of $(X, Y)$, after requesting
	\begin{equation}
	\label{eq:samplecomplexity}
	n = O\left(\frac{d\; \theta(\varepsilon/p)}{p^2} \log^2 
	\left(\frac{d}{p\varepsilon\delta}\right) \right),
	\end{equation}
	labels it produces a $\{0,1,*\}$-valued classifier $\widehat{f}_p$ satisfying , with 
	probability at least $1 - \delta$,
	\[
	R^p(\widehat{f}_p) - \inf\limits_{f \in \F}R(f) \le \varepsilon.
	\] Here $\theta(\cdot)$ is the disagreement coefficient, $d$ is the VC dimension of 
	a $\{0,1\}$-valued class $\F$.
\end{Theorem}
A formal statement of this result is Theorem \ref{thm:sample_comp_abs}.
Observe that if the disagreement coefficient $\theta(\cdot)$ is bounded, which holds, 
for example, for threshold classifiers on the real line and  homogeneous linear separators 
in $\mathbb{R}^d$
under a uniform distribution on the unit sphere \citep{hanneke14}, then exponential savings are always 
possible by the above result. Indeed, in this case, the sample complexity bound 
\eqref{eq:samplecomplexity} scales as $O\left(\log^2 
\left(\frac{1}{\varepsilon}\right)\right)$. 
Also, the definition of $\theta(\cdot)$ implies that the dependence on $\varepsilon$ in 
\eqref{eq:samplecomplexity} is never worse than $O\left(\frac{1}{\varepsilon}\log^2\left( 
\frac{1}{\varepsilon}\right)\right)$. This is superior to the passive 
$\Theta(\frac{1}{\varepsilon^2})$ sample complexity. To be more specific, we illustrate 
our result by the following basic example.
\begin{Example}
	For threshold classifiers on the real line, if the price of abstention 
	is $1/2 - p = 0.49$, our result implies that for \emph{any distribution} of the data,  
	$O\left(\log^2\left(\frac{1}{\eps}\right)\right)$ label requests are sufficient to 
	guarantee that $R^p(\widehat{f}_p) - \inf_{f \in \F}R(f) \le \varepsilon$. If either 
	abstention is not allowed or $p = 0$, the number of label requests 
	$\Theta(\frac{1}{\eps^2})$ cannot be generally improved in the active learning setup.
\end{Example}
The reader can recall that a similar sample complexity bound holds in the bounded noise 
model of \citet{Massart06}, that is, when the \emph{Bayes optimal classifier} $f^*_B$ 
belongs to $\F$ and $|2\Pr(Y = 1|X) - 1| \ge h > 0$ almost surely (see, for example, 
\citep[Section 7.1]{hanneke2015minimax}). And at least on an intuitive level, when 
$f^*_B \in \F$, an option to abstain can be potentially used to eliminate the 
noise and reduce the problem to the bounded noise case. More importantly, 
our result is also robust to the model misspecification, that is, we allow  $f^*_B \notin \F$. 
Indeed, according to \cite{Kaariainen05}, the model misspecification alone can result in 
the 
$\Omega(\frac{1}{\varepsilon^2})$ lower bound even if there is no noise in the labeling 
mechanism. Therefore, our estimator with a reject option avoids \emph{both known 
reasons} of $\Omega(\frac{1}{\varepsilon^2})$ lower bounds.

Our second result is the minimax analysis of the standard active learning 
setup. We exploit our classifier with a reject option as an intermediate step. By the 
minimax analysis, we usually mean the sample complexity bounds valid for any marginal 
distribution (denoted by $P_X$) of the unlabeled data. In this setup, an aforementioned 
lower bound of 
\cite{Kaariainen05} implies that one should restrict the noise of the problem to get 
exponential 
savings, that is, we assume 
\begin{equation}
\label{eq:massnoise}
|2\Pr(Y = 1|X) - 1| \geq h\ \text{almost surely for some}\ h > 0.
\end{equation}
We answer the following question.
\vskip+5pt
\begin{tcolorbox}
	Assuming Massart's noise \eqref{eq:massnoise} what is the characterization of 
	$\F$ allowing exponential savings in active learning? 
\end{tcolorbox}
\vskip+5pt
Under a strong assumption that the Bayes optimal classifier $f^*_B$ is in $\F$, this 
question has been answered in \citep[Theorem 4]{hanneke2015minimax}: exponential 
savings are possible under $\eqref{eq:massnoise}$ and $f^*_B \in \F$ if and only if the 
\emph{star number} $\mathbf{s}$ is finite (defined in Section \ref{sec:notation}). We 
need to define the 
\emph{diameter} of $\F$. It is the smallest integer $D$ such that
\[
\sup_{f, g \in \F}|\{x \in \X: f(x) \neq g(x)\}| \leq D,
\]
where $\X$ is our instance space (see Section \ref{sec:notation}).

\begin{Theorem}[An informal statement]
	Exponential savings are possible in active classification for any distribution 
	satisfying Massart's noise assumption \eqref{eq:massnoise} (without assuming 
	$f^*_B \in \F$)  \emph{if and only if} both the star number $\mathbf{s}$ (or 
	respectively the disagreement coefficient $\theta(\cdot)$ if the dependence on 
	$P_X$ is allowed) and the diameter $D$ are finite.   
\end{Theorem}
A formal version of this result is Theorem \ref{thm:fin_dim}.
As we mentioned, one may show that Massart's assumption \eqref{eq:massnoise} is also 
inevitable if one wants to have exponential savings for any marginal distribution $P_X$ of 
the unlabeled data. We discuss this in more detail in Section 
\ref{sec:combchar}.

The disagreement coefficient and the star number of 
\citet{hanneke2007bound,hanneke2015minimax} play an important role in the 
active learning literature while the diameter of \citet{Bendavid14} is used in the analysis 
of passive learning with \emph{deterministic labeling}, that is, when $Y = f^*_B(X)$ 
almost surely. It appears that both the star number and the diameter are infinite in many 
natural scenarios. 

\subsection{Our contributions}
\begin{itemize}
	\item In Section \ref{sec:expratesabst}, we present our main result as well as the 
	performance bound for a passive algorithm called the \emph{mid-point algorithm}. 
	\item In Section \ref{sec:adaptivity}, we show the adaptivity of our results to the 
	bounded noise assumption. In 
	particular, our algorithm abstains rarely when this assumption holds.  
	\item In Section \ref{sec:combchar}, we return to the standard active learning setup 
	where the reject option is not available. We characterize the case where the Bayes 
	optimal rule is not in the class, the Bernstein assumption is vacuous, but exponential 
	savings 
	are still possible.
\end{itemize}

\subsection{Related work}

We start with a concise literature overview followed by a more detailed comparison with 
some related recent results.

The most standard algorithm in the realizable case is referred to as the CAL algorithm 
(after 
the names of \citet*{cohn1994improving}). This algorithm can be shown to 
provide exponential savings in some cases.  The analysis of the realizable case with the 
complexity measure depending on the marginal distribution of the unlabeled data is by 
\citet{dasgupta2005coarse}. In particular, \citeauthor{dasgupta2005coarse} 
generalizes various examples of exponential savings in realizable active learning. 
The fact that exponential savings are also possible in the bounded noise case for the 
threshold functions is attributed to  \citet{burnashev1974interval}; their ideas were later 
developed in \citep{korostelev1999minimax, golubev2003sequential, castro2008minimax}. The first general agnostic active learning algorithm is presented in \citep*{balcan2009agnostic} followed by a more refined analysis in 
\citep{hanneke2007bound, dasgupta2008general, beygelzimer2009importance, Hsu10, koltchinskii2010rademacher, hanneke2011rates, raginsky11, zhang2014beyond, hanneke2015minimax} and other works. Most of the known upper bounds are based on the \emph{disagreement coefficient} introduced by \citet{hanneke2007bound} to 
analyze the performance of active learning algorithms; essentially the same quantity also 
appeared in the analysis of ratio-type empirical processes \citep{alexander1987rates} 
and was later reintroduced to the passive learning literature by
\citet{gine2006concentration}. We refer to the survey \citep{hanneke14} for 
a detailed exposition of these results.

The risk \eqref{rp} was analyzed in the seminal work of \citet{Chow70}. The statistical 
analysis in the context of passive learning was first provided in \citep{Wegkamp06, 
	Bartlett08}. The authors consider the reject option as an action available not only to 
	the 
learner but also to the classifiers in the base class.
For a more extensive survey and some related results, we refer to \citep{Freund04, 
Elyaniv10, cortes2016learning, yan2016active}, and \citep{Elyaniv17}.
Recently \citet{bousquet2019fast, neu2020} show that if the learner is given an 
option to abstain, and the risk of Chow \eqref{rp} is used, then the so-called \emph{fast 
rates} are 
possible without additional assumptions in passive and online classification. There is 
also a line of research devoted to active learning in the 
non-parametric setup not covered in this paper \citep{castro2008minimax, 
	koltchinskii2010rademacher, minsker2012plug, locatelli17a, locatelli18a}. An 
	extension of 
our results to the non-parametric setup is the natural direction of future work. 

In the context of active learning, Chow's risk has been recently analyzed in 
\citep{shekhar2020active}. The authors consider a non-parametric classification 
problem and 
make some margin-type assumptions, which is different from our setup. Chow's risk is 
also connected to surrogate losses appearing in the context of active 
learning in \citep{hanneke2019surrogate}, where the main purpose of using these losses 
is to simplify the computational problems associated with minimizing the binary loss. In 
particular, their statistical results are always not better than for the binary loss, 
which is again different from our findings.

\medskip\noindent
\paragraph{Relations to \citep{bousquet2019fast}.}
The model we are considering has been recently considered in the context of passive 
learning. \citeauthor{bousquet2019fast} show that the passive learning sample 
complexity $\Theta(\frac{1}{\varepsilon^2})$ can be improved to 
$O(\frac{1}{p\varepsilon}\log \frac{1}{\varepsilon})$ whenever \eqref{eq:rpexcessrisk} is 
used.
We use an improved version of their argument as a subroutine and provide a 
simplified analysis for it. It appears that in the context of active learning just being able 
to provide a $O(\frac{1}{p\varepsilon}\log \frac{1}{\varepsilon})$ sample complexity is not 
sufficient. We instead 
directly exploit a phenomenon first observed by \citet{Audibert07}, which in our case 
can be described as follows: for some realizations of the labels the difference 
$R^p(\widehat{f}) - R(f^*)$ can be negative and 
\eqref{eq:rpexcessrisk} immediately follows. One of our key technical observations is 
that there 
is a way to detect this event using only the learning sample.

\medskip\noindent
\paragraph{Relations to selective classification.}
In selective classification, a learner aims to provide a pair of $\{0, 1\}$-valued 
functions $(\widehat{f}, \widehat{g})$ called a selective classifier.
The classifier $\widehat f(x)$ must be pointwise competitive, that is, $\widehat f(x) = 
f^*(x)$, where $f^* = \arg\min_{f \in \F} R(f)$ if and only if $\widehat{g}(x) = 1$.
If $\widehat g(x) = 0$ the selective classifier abstains.
The learner is interested in minimizing the rejection rate $\Pr(\widehat g(X) = 0)$.
The pointwise competitiveness requirement is quite restrictive and is not always 
desirable, especially for those instances where $f^*$ predicts differently from the Bayes 
optimal
classifier $f^*_B$.
To the best of our knowledge, only the realizable case \citep{Elyaniv10, el2012active} 
and the case 
of a small $R(f^*)$ \citep{Elyaniv17} were considered so far.
In this paper, we are pursuing a less ambitious goal and allow our classifier to make a 
small 
portion of mistakes when it does not abstain. As a result, our improvements are 
somewhat more substantial. 

\medskip\noindent
\paragraph{Relations to the minimax analysis of \citet{hanneke2015minimax}.}
The work of \citet{hanneke2015minimax} provides an almost complete picture of the 
sample complexity bounds in the minimax sense under the bounded noise and the
Bernstein assumption. It remains open if these savings are 
possible in other cases. We make one step forward and 
show that these improvements can be obtained for some classes even if the problem is
misspecified, that is, the Bayes optimal rule is not in the class, and the Bernstein 
condition is vacuous.   

\medskip\noindent
\paragraph{Relations to confidence-rated predictors.}
\citet{zhang2014beyond} analyze the disagreement based active learning algorithms via 
the confidence-rated predictor: instead of requesting a label of a specific point in the 
disagreement set, their algorithm can randomly abstain from doing so. 
Similar techniques were used in \citep{balcan2007margin, balcan2013active} for linear 
separators. This approach leads to some improvements in the sample complexity bounds under various low noise assumptions. However, our 
analysis uses an option to abstain only when classifying some of the problematic 
instances.

\section{Notation and Setup}
\label{sec:notation}
We introduce some notation and basic definitions that will be used throughout the text. 
The symbol $\Ind[A]$ denotes an indicator function of the event $A$. The notation $f 
\lesssim g$ or $g \gtrsim f$ means that for some universal constant $c>0$ we have $f 
\le cg$. To avoid the problems with the logarithmic function we assume that $\log x$ 
means $\max\{\log x, 1\}$. Throughout the paper we also use the standard $O(\cdot), 
\Omega(\cdot), \Theta(\cdot)$ notation.
We set $a \wedge b = \min\{a, b\}$ and $a \lor b = \max\{a, b\}$. For $s \ge 1$ we define 
the $L_s(P)$ norm as $\|g\|_{L_s} = \left(\E 
|g(Z)|^s\right)^\frac{1}{s}$, where the expectation is taken with respect to some measure 
$P$ always clear from the context. The 
$L_s(P)$ diameter of $\F$ is 
\[
\D(\F, L_s) = \sup_{f, g \in \F}\|f - g\|_{L_s}.
\]

We define the instance space $\mathcal{X}$ and the label space $\mathcal{Y} = \{0, 
1\}$. We assume that the set $\mathcal{X} \times \mathcal{Y}$ is equipped with some 
$\sigma$-algebra and a probability measure $P = P_{X, Y}$ on 
measurable subsets is defined. We also assume that we are given a set of classifiers 
$\mathcal F$ mapping $\mathcal{X}$ to $\mathcal{Y}$. \emph{In passive learning} 
we observe $S_n = \left\{(X_{1}, Y_{1}), \ldots, (X_{n}, Y_{n})\right\}$ sampled according 
to $P$.
In the \emph{pool-based active learning}, we define an active learning algorithm as an 
algorithm taking as input a budget $n \in \mathbb{N}$, and proceeding as follows. The 
algorithm
initially uses an unlabeled infinite data sequence $X_1, X_2, \ldots$ distributed 
according to $P_X$. The algorithm may select an index $i_1$ and request the label 
$Y_{i_1}$. In this case we observe the value of $Y_{i_1}$, sampled according to the 
conditional distribution $Y|X_{i_1}$; then based on both the unlabeled sequence and 
$Y_{i_1}$, it may select another index $i_2 > i_1$ and request to observe $Y_{i_2}$. This 
continues for at most $n$ rounds. Finally the algorithm outputs a classifier 
$\widehat{f}$.

Given $S_n = \left\{(X_{1}, Y_{1}), \ldots, (X_{n}, Y_{n})\right\}$ let $P_{S_n}$ denote the expectation (as well as the empirical measure) with respect to the empirical measure induced by this 
sample. We sometimes write $P_n f$ instead of $P_{S_n} f(X)$ and $Pf$ instead of $\E 
f(X)$.
For a set $\{x_1, \ldots, x_k\} \subseteq \mathcal{X}$ and a class of $\{0, 1\}$-valued 
functions $\F$, we denote the 
restriction of $\F$ on $\{x_1, \dots, x_k\}$ by $\F_{\{x_1, \ldots, x_k\}} = \{(f(x_1), \dots, 
f(x_k)) : f \in \F\}$ .
The value of the growth function $\mathcal S_\F(k)$ is defined as the largest cardinality 
of $\F_{\{x_1, 
	\ldots, x_k\}}$ among all $x_1, \dots, x_k \in \mathcal X$.
The VC dimension of $\mathcal F$ is the largest integer $d$ such that $\mathcal S_\F(d) 
= 2^d$ \citep{Vapnik68}.
For any set $\F$ of classifiers let the disagreement set of $\F$ be defined as

\[
    \dis(\F) = \{x \in \X:\; \exists \; f, g \in \F \; \text{such that}\; f(x) \neq g(x)\}.
\]

As above, define the prediction risk as $R(f) = \Pr(f(X) \neq Y)$ and the Chow's risk 
$R^p$ 
risk is given by \eqref{rp}. The Bayes optimal rule $f^*_{B}$ and the best classifier in the 
class $f^*$ are given by
\[
    f^*_{B}(x)=\Ind[\Pr(Y=1| X = x)\ge 1/2]
\]
and
\[
    f^* = \argmin_{f \in \F}R(f).
\]
The largest $h \ge 0$ such that almost surely
\begin{equation}
\label{eq:massartsmargin}
|2\Pr(Y=1|X)-1| \ge h
\end{equation}
is called Massart's margin parameter \citep{Massart06}. Let 
\[
R_{S_n}(f) = \frac{1}{n}\sum_{(X_i, Y_i) \in S_n}\Ind[f(X_i) \neq Y_i]
\]
denote the empirical risk with 
respect to $S_n$.
We sometimes write $R_n(f)$ instead of $R_{S_n}(f)$ when the sample is clear from the 
context.
Any minimizer of the empirical risk $R_{S_n}(f)$ in $\F$ is called ERM.
For a $\{0, 1, *\}$-valued classifier $g$ we define the empirical Chow's risk as
\begin{align*}
    R^p_{S_n}(g)
    &
    = \frac{1}{n}\sum\limits_{(X_i, Y_i) \in S_n}\Ind[g(X_i) \neq Y_i\ \text{and}\ g(X_i) \in \{0, 1\} ]
    \\&
    \quad
    + \frac{1/2 - p}{n}\sum\limits_{(X_i, Y_i) \in S_n}\Ind[g(X_i) = *].
\end{align*}
Fix $\varepsilon \ge 0$.
The disagreement coefficient $\theta(\cdot)$ of \citet{hanneke2007bound} is defined as
\[
	\theta(\eps) = \sup\limits_{g \in \F, \eps_0 \ge \eps}\frac{P_X(\dis(\{f \in \F: \|f - g\|_{L_1} 
	\le \eps_0\}))}{\eps_0} \lor 1.
\]
\begin{Remark}
	The definition of $\theta(\cdot)$ yields that, for any class $\F$,
	\begin{align}
	\label{remark_theta}
		\theta(\D(\F, L_1(P))) & 
		\geq \frac{P_X(\dis(\{f \in \F: \|f - g\|_{L_1} \le \D(\F, L_1(P))\}))}{\D(\F, L_1(P))}
		\\&\notag
		= \frac{P_X(\dis(\F))}{\D(\F, L_1(P))}.
	\end{align}
\end{Remark}

Finally, the star number of \citet{hanneke2015minimax} is the largest integer 
$\mathbf{s}$ such that there exist $f_0, f_1, \ldots, f_{\mathbf{s}} \in \F$ and $x_{1}, 
\ldots, x_{\mathbf{s}} \in \X$ such that for all $i \in \{1, \ldots, n\}$,
\[
\dis(\{f_0, f_i\}) \cap \{x_{1}, 
\ldots, x_{\mathbf{s}}\} = \{x_i\}.
\]
\section{Active learning with abstention}
\label{sec:expratesabst}
We present our main result, which is slightly sharper than our simplified statement in 
Section \ref{sec:introduction}.
\begin{Theorem}
	\label{thm:sample_comp_abs}
	Fix $\eps, \delta \in (0,1], p \in (0, \frac{1}{2}]$. There are problems with definition of the number of iterations if $p=0$.) Assume that the VC dimension of 
	$\F$ is equal to $d$.
	There is an active learning 
	algorithm (namely, Algorithm \ref{alg_al_abs}) such that after requesting at most
	\begin{equation}
	    \label{eq:sample_comp_abs}
	    n = O\left(\frac{\theta(\eps/p)}{p^2} \left( d \log^2\left(\frac{d}{p\eps} \right) + \log^2\left(\frac{1}{\delta}\right) \right)\right)
	\end{equation}
	labels, it returns, with probability at least $1 - \delta$, a classifier $\widehat f_p$ 
	satisfying 
	\[
	    R^{p}(\widehat f_p) - R(f^*) \leq \eps.
	\]
\end{Theorem}
The proof of this result is included in Appendix \ref{sec:sample_comp_abs_proof}. We are ready to present the 
algorithm achieving these guarantees. 
\vskip+5pt
\begin{tcolorbox}[breakable, enhanced]
	\begin{myalgorithm}
		\label{alg_al_abs}
		\hfill
		\begin{itemize}
			\item Let $V_0 = \F$ and
			\begin{equation}
			\label{eq:alpha}
    			\hskip-10pt
    			\alpha^2(n, \delta)
    			= \frac4n \left( 3d\log \frac{e(2n \vee d)}d
    			+ \log\frac{56}\delta \right),
			\end{equation}
			and set
			\[
			    \hskip-10pt
			    J = \min\left\{ k \in \mathbb N : 148 \alpha^2(2^{k-1}, \delta/(k+1)^2) / p \leq \eps \right\}.
			\]
			\item \textbf{for} $j$ from $1$ to $J$ \textbf{do}
			\begin{enumerate}
				\item Sample $n_j = 2^{j-1}$ fresh i.i.d. instances $X_{2^{j-1}}, \dots, X_{2^{j} - 1}$ from $P_X$ and denote them by $Q_j = \{X_{2^{j-1}}, 
				\dots, X_{2^{j} - 1}\}$.
				\item Define $D_j = \dis(V_{j-1}) \cap Q_j$.
				\item Request labels for all instances in $D_j$.
				\item Set $S_j = \bigcup\limits_{X_m \in D_j} \{(X_m ,Y_m)\}$.
				\item  Compute (any) $\mathrm{ERM}$ $\widehat f_j \in 
				\argmin\limits_{f \in V_{j-1}} 
				R_{S_j}(f)$.
				\item  For  $n_j = 2^{j-1}$ and $\delta_j = \delta / (j+1)^2$, update
				\begin{align*}
				\hskip-45pt	V_j = \Big\{
				&
				f \in V_{j - 1} :
				\frac{|S_j|}{n_j} \left( R_{S_j}(f) - 
				R_{S_j}(\widehat f_j) \right)
				\\&
				\leq 2\alpha^2(n_j, \delta_j) + 2\alpha(n_j, \delta_j) \sqrt{P_{Q_j}|f - \widehat f_j|} \Big\}.
				\end{align*}
				\item \textbf{if} $\D(V_j, L_2(P_{Q_j})) > 49 \alpha(n_j, \delta_j)/p$ 
					\textbf{or} $j = J$,  
				\begin{itemize}
					\item  consider the class $\widehat \G_j = \left\{ \frac{f + \widehat f_j}2 : f \in V_j \right\}$ of $\{0, 1, 1/2\}$-valued functions and convert it into 
					$\{0, 1, 
					*\}$-valued class $U_j$ by replacing $1/2$ with $*$;
					\item define the \emph{mid-point classifier} as (also Algorithm 
					\ref{alg_al_abs_midp} below) 
					\[
					\widehat f_p \in \argmin\limits_{f \in U_j} R^{p}_{S_j}(f).
					\]
					\item \textbf{return} $\widehat f_p$.
				\end{itemize}
			\end{enumerate}
			\item \textbf{end for}
		\end{itemize}
	\end{myalgorithm}
\end{tcolorbox}
\vskip+5pt
Let us discuss the mechanism behind this algorithm. At iteration $j$, our strategy 
maintains a set $V_j$ of candidate classifiers and requests the labels of instances that 
belong to the disagreement set of $V_j$. Then we update the set $V_j$ by removing all
classifiers making a large number of mistakes on the requested labels. This part of our 
algorithm is standard and corresponds to the principle standing behind all 
disagreement-based algorithms. Our first modification is that at each iteration we also
compute the empirical  diameter $\D(V_j, L_2(P_{Q_j}))$ and compare it with the 
threshold value $49 \alpha(n_j, \delta_j)/p$. It follows that a large value of $\D(V_j, 
L_2(P_{Q_j}))$ indicates that the current iteration is too \say{noisy} and the reject option 
can help. Otherwise, if $\D(V_j, L_2(P_{Q_j}))$ is small, we proceed with the standard 
active learning strategy described above. Observe that $\D(V_j, L_2(P_{Q_j})) \leq 1$, but 
$49 \alpha(n_j, \delta_j)/p$ is always greater than $1$ for small values of $j$ so that we 
never return $\widehat f_p$ too early. Our second modification is that $\widehat f_p$ is 
built using a two-step aggregation procedure described in detail in Section
\ref{sec:strategy}. 

In Section \ref{sec:adaptivity} we show that Algorithm \ref{alg_al_abs} is adaptive to the 
favorable noise assumptions: under the bounded noise assumption the event $\D(V_j, 
L_2(P_{Q_j})) \geq 49 \alpha(n_j, \delta_j)/p$ almost never happens, $\Pr(\widehat f_p(X) 
= *)$ is small, and our algorithm mimics the behavior of the standard active learning 
strategy such as, for example, the one of \citet*{dasgupta2008general}.

\subsection{Strategy of the proof and the mid-point algorithm}
\label{sec:strategy}

In this section, we introduce the \emph{mid-point algorithm} used in Algorithm 
\ref{alg_al_abs}.
This algorithm is a simplified version of the aggregation procedure in 
\citep{bousquet2019fast}, inspired in turn by several key aggregation algorithms 
for the squared loss \citep{Audibert07, Lecue09, Mendelson18}.
\vskip+5pt
\begin{tcolorbox}[breakable, enhanced]
	\begin{myalgorithm}[Mid-point Algorithm]
		\label{alg_al_abs_midp}
		\hfill
		\begin{itemize}
			\item Given the labeled sample $S_n$ and the class $\F$ and the 
			confidence $\delta$ and the abstention margin $p \in (0, \frac{1}{2}]$.
			\item Find (any) $\mathrm{ERM}$
			$
			\widehat g \in \argmin\limits_{f \in \F} R_{S_n}(f).
			$
			\item Let $\alpha(n, \delta)$ be as in \eqref{eq:alpha} and define 
			\begin{align*}
			    V
			    = \Big\{
			    &
			    f \in \F: R_{S_n}(f) - R_{S_n}(\widehat g)
			    \\&
			    \leq 2\alpha^2(n, \delta) + 2\alpha(n, \delta) \sqrt{P_{n}|f - \widehat g|} \Big\}.
			\end{align*}
			\item Consider the (random) set $\left\{ \frac{f + \widehat g}2 : f \in 
			V \right\}$ of $\{0, 1, 1/2\}$-valued functions and convert it into $\{0, 1, 
			*\}$-valued set $\widehat \G$ by replacing $1/2$ with $*$.
			\item  Define the \emph{mid-point classifier} as
			\[
			    \widetilde f_p \in \argmin\limits_{f \in \widehat \G} R^{p}_{S_n}(f).
			\]
			\item \textbf{return} $\widetilde f_p.$
		\end{itemize}
	\end{myalgorithm}
\end{tcolorbox}
\vskip+5pt
We are ready to provide a data-dependent bound for this algorithm.
\begin{Theorem}
	\label{thm:empirical_midpoint}
	Fix $p \in (0, \frac{1}{2}], \delta\in (0,1)$. Assume that the VC dimension of $\F$ is 
	equal to $d$. 
	In the notation of Algorithm \ref{alg_al_abs_midp}, we have that, with probability at 
	least $1 - \delta$, $f^* \in V$ and 
	\begin{align}
	\label{eq:datadepbound}
	R^p(\widetilde f_p) - R(f^*)
	\leq 8 \alpha^2(n, \delta)
	&\notag
	+ 12 \alpha(n, \delta) \D(V, L_2(P_n))
	\\&
	- \frac p4 \D^2(V, L_2(P_n)),
	\end{align}
	where $\alpha(n, \delta)$ is given by \eqref{eq:alpha}.
	In particular, on this event whenever $\D(V, L_2(P_n)) \geq 49 \alpha(n, \delta)/p$, 
	we 
	have
	\begin{equation}
	\label{eq:belowzero}
	R^p(\widetilde{f}_p) < R(f^*).
	\end{equation}
\end{Theorem}

The property of the Mid-point algorithm that it has a negative excess $R^p$-risk in some situations plays a crucial role in the analysis of Algorithm \ref{alg_al_abs}.
This phenomenon happens because of the property of $R^p$-risk.
If there are two functions $f, g \in \F$ that disagree too often but have close empirical risks, then it is better to abstain on their disagreement set rather than request additional labels. It appears that in this case, the price of abstention becomes smaller than the possible gain from finding the best classifier among $f$ and $g$.
For the rest of this section, we discuss how this passive learning result is used in the 
proof of Theorem 
\ref{thm:sample_comp_abs}. The second part of the statement of Theorem 
\ref{thm:empirical_midpoint} is one of our main technical insights. This result means that 
for any labeled sample of size $m$ we may compute the data-dependent value $\D(V, 
L_2(P_m))$ and if it is larger than $49 \alpha(m, \delta)/p$, we conclude that 
$\widetilde{f}_p$ outperforms $f^*$. This is a favorable 
scenario in our context.
Our second observation is that if $\D(V, L_2(P_m))$ is smaller than $49 \alpha(m, 
\delta)/p$, then one may show that the region of disagreement of $V$ is small. 
Indeed, by the definition of $\theta(\cdot)$ and the uniform convergence, one can show 
that
\begin{align*}
    P(\dis(V))
    &
    \le \theta(\D^2(V, L_2(P)))\D^2(V, L_2(P))
    \\&
    \approx \theta(\D^2(V, L_2(P_m)))\D^2(V, L_2(P_m)).
\end{align*}
Here the first inequality follows from \eqref{remark_theta} and the fact that, for any set $V$ of $\{0, 1\}$-valued functions, it holds that 
\begin{align*}
	\D^2(V, L_2(P))
	&
	= \sup\limits_{f, g \in V} \E |f(X) - g(X)|^2
	\\&
	= \sup\limits_{f, g \in V} \E |f(X) - g(X)|
	= \D(V, L_1(P)).
\end{align*}
From this moment on, we use a standard active learning analysis following closely the 
well-proven techniques of \citet*{dasgupta2008general} (see also \citep{Hsu10, 
zhang2014beyond}). In some sense, our analysis reveals a dichotomy: at each iteration 
of Algorithm \ref{alg_al_abs}, we have that either the noise of the problem is so high that 
even the 
negative excess risk is possible through the reject option, or the problem is as 
good as if the bounded noise assumption holds. Moreover, both situations can be empirically detected. As we pointed out, the negativity of the excess risk (regret) 
for improper learners in passive (online) learning as in \eqref{eq:belowzero} is not well 
understood. Among the few works exploring this is the paper of \citet{Audibert07} where 
the negativity of the excess risk is used to explain why the so-called progressive mixture 
rules are deviation suboptimal. More recently, \citet*{mourtada2021distribution} used 
the negativity of the excess risk to observe the same suboptimality for truncated linear 
least squares. The analysis of this paper reveals that the negativity of the excess risk is 
helpful in active learning.

\begin{Remark}
	\label{rem:prevres}
	Maximizing \eqref{eq:datadepbound} with respect to $\D(V, L_2(P_n))$, we have, 
	with probability at least $1 - \delta$,
	\[
	R^p(\widetilde{f}_p) - R(f^*)
	\lesssim \frac{d\log(n/d) + \log(1/\delta)}{np}.
	\]
	This bound is achieved in \citep[Theorem 2.1]{bousquet2019fast} by an algorithm 
	requiring an additional sample splitting step.
\end{Remark}

We defer the proof of Theorem \ref{thm:empirical_midpoint}
to Appendix \ref{sec:empmidpointproof}.

\subsection{Adaptation to the bounded noise assumption}
\label{sec:adaptivity}

Assume that Massart's noise condition \eqref{eq:massartsmargin} holds with $h > 0$ 
and that the Bayes rule $f^*_B$ belongs to 
the class $\F$.
Under these conditions, there are active learning algorithms (see, for example,
\cite{zhang2014beyond, hanneke2015minimax}) showing exponential savings in active 
learning whenever the disagreement coefficient is bounded.
We show that Algorithm \ref{alg_al_abs} adapts to the bounded noise condition in the 
sense 
that if $p \leq h/4$, it also provides exponential savings and outputs a classifier 
$\widehat f_p$ such that $\Pr( 
\widehat f_p(X) = *)$ is small. 

It is known (see, for example, \citep[Equation (5)]{Wegkamp06}) that the optimal 
$\{0, 1, *\}$-valued classifier with 
respect to Chow's risk \eqref{rp} is given by
\[
f^*_p(x) =
\begin{cases}
f^*_B(x), \quad \text{if } |2\Pr(Y = 1 | x) - 1| \geq 2p,\\
*, \quad \text{otherwise}.
\end{cases}
\]
We see that if Massart's noise condition holds, the Bayes rule $f^*_B$ minimizes the risk 
\eqref{rp} for all $p \leq h/2$. Therefore, if $p \leq h/2$ and $f^*_B \in \F$, the excess risk 
$R^p(\widetilde f_j) - R(f^*_B)$ cannot be negative.
Algorithm \ref{alg_al_abs} is constructed in such a way that if it terminates before the 
$J$-th iteration, then
the excess risk $R^p(\widetilde f_j) - R(f^*)$ is negative  (see the details of the proof 
of Theorem \ref{thm:sample_comp_abs}).
This yields that Algorithm \ref{alg_al_abs} finishes after $J$ iterations in our case. We 
also have the following result.
\begin{Proposition}
	\label{prop:adaptivity}
	In the notation of Theorem \ref{thm:sample_comp_abs}, assume that the noise 
	condition \eqref{eq:massartsmargin} with the
	parameter $h > 0$ holds and that $f^*_B \in \F$.
	We have for $p \in (0, h/4]$ that the output classifier $\widehat f_p$ of 
	Algorithm \ref{alg_al_abs}, with the number of label requests
	\[
		n = O\left(\frac{\theta(\eps/p)}{p^2} \left( d \log^2\left(\frac{d}{p\eps} \right) + \log^2\left(\frac{1}{\delta}\right) \right)\right)
	\]
	as in Theorem \ref{thm:sample_comp_abs}, satisfies, with probability at least $1 - \delta$,
	\[
	\Pr(\widehat f_p(X) = * ) \leq 4\eps / h.
	\]
	Moreover, on the same event, if all $*$-s are replaced by random guessing, which corresponds to the risk $R^0$, we also have
	\begin{equation}
	\label{eq:excessriskwithrandomguess}
	    R^0(\widehat f_p) - R(f^*) \leq 2\eps. 
	\end{equation}
\end{Proposition}

Proposition \ref{prop:adaptivity} indicates that the dependence on $p$ in Theorem \ref{thm:sample_comp_abs} (and consequently in Proposition \ref{prop:adaptivity}) is captured correctly up to some logarithmic factors.
Indeed, let Massart's margin parameter $h > 0$ be at most $1/2$.
In \citep[Theorem 2]{raginsky11}, the authors proved that, for any active learning algorithm (including the algorithms with external randomization such as the one in \eqref{eq:excessriskwithrandomguess}), there exist a distribution over $\X \times \{0, 1\}$, satisfying \eqref{eq:massnoise}, and a class $\F$ with $\text{VCdim}(\F) = d$ such that $f_B^* \in \F$ and the algorithm needs \begin{equation}
    \label{eq:raginsky_th2}
	\Omega\left( \frac{d \log\theta(\eps)}{h^2} + \frac{\theta(\eps) \log(1/\delta)}{h^2} \right)
\end{equation}
label requests to get the excess risk at most $\eps$ with probability at least $1 - \delta$.
For example, if one assumes that $p^2$ in \eqref{eq:sample_comp_abs} can be replaced by $p^\alpha$ with some $\alpha \in (0, 2)$, then the label complexity of $\widehat f_p$ with all $*$-s replaced by random guessing and $p = h/8$ will be equal to
\[
	O\left(\frac{\theta(\eps/h)}{h^\alpha} \left( d \log^2\left(\frac{d}{h\eps} \right) + \log^2\left(\frac{1}{\delta}\right) \right)\right),
\]
which contradicts the lower bound \eqref{eq:raginsky_th2}. The same reasoning shows that our bounds cannot be significantly improved with respect to both $\theta(\eps)$ and $\delta$. Further, the lower bound \citep[Theorem 3]{hanneke2015minimax} for the realizable case, applied to the class of thresholds, yields that, for any active learning algorithm, there exists a distribution $P_X$ such that the label complexity of the algorithm is $\Omega(\log(1/\eps))$.
Since the realizable case is a particular case of Massart's noise with $h=1$, the logarithmic dependence on $\eps$ cannot be completely removed.

However, there is still a small room for improvement.
In \citep[Theorem 4]{hanneke2015minimax}, the authors showed that if a class $\F$ has a finite star number $\mathbf{s} < \infty$, then there is an active learning algorithm with the label complexity
\begin{equation}
\label{eq:starnumberbound}
    O\left( \frac{\mathbf s}{h^2} \polylog\left( \frac{d}{\eps\delta} \right) \right)
\end{equation}
in the presence of Massart's noise, provided that $f_B^* \in \F$.
Hence, in the case of finite star number, the product $\theta(\eps/p) d$ in the upper bound can be replaced by $\mathbf{s}$ rather than by $\mathbf{s} d$ following from our analysis (by \citep[Theorem 10]{hanneke2015minimax} we have $\theta(\eps/h) \le \mathbf{s}$ for 
any distribution $P_X$ of the unlabeled data).
A question, whether the improved rate \eqref{eq:starnumberbound} can be achieved in our setup, is open.
Summing up, there are some gaps between the state-of-the-art upper and lower bounds on the label complexity in active learning in the presence of Massart's noise, and these questions are also relevant in our setup.

With minor efforts, a similar result is achievable if instead of assuming that $f^*_B \in \F$ and \eqref{eq:massartsmargin} holds, we have that the 
\emph{Bernstein assumption} holds. That is, for any $f \in \F$,
\[
h\Pr(f(X) \neq f^*(X)) \le R(f) - R(f^*).
\]
We omit these derivations in favor of a more transparent Proposition 
\ref{prop:adaptivity}. The proof of Proposition \ref{prop:adaptivity} reveals that in the 
case where $f^*_B \in 
\F$, our passive Algorithm \ref{alg_al_abs_midp} abstains most of the time
only on the instances where Massart's noise assumption is not satisfied.   This observation can be useful in the context of selective 
classification described above.
\begin{Proposition}
	\label{prop:secprop}
	Assume that $f^*_B \in \F$. Then, the classifier $\widetilde f_p$ of passive 
	Algorithm \ref{alg_al_abs_midp} trained on $S_n$ satisfies, with probability at least 
	$1 - \delta$, 
	\begin{align*}
	    &
	    \Pr\left( \widetilde f_p(X) = *\ \mathrm{ and }\ |2\Pr(Y = 1|X) - 1| \geq 4p  \right)
	    \leq \frac{592}{np^2} \left( 3d \log \frac{e(2n \vee d)}d + \log \frac{56}\delta \right).
	\end{align*}
\end{Proposition}

\section{Exponential savings under the model misspecification}
\label{sec:combchar}

We return to the setting where abstention
is not allowed. The result of \citet[Theorem 4]{hanneke2015minimax} implies that if 
Massart's noise assumption \eqref{eq:massartsmargin} holds, $\eps \in (0, h/24), \delta 
\in [0, 1/24]$, and the Bayes optimal rule $f^*_B$ belongs to  $\F$, then at least 
\begin{equation}
\label{eq:activelowerbound}
\Omega\left(\frac{1}{h^2}\left((1 - h)\min\left\{\mathbf{s}, \frac{h}{\eps}\right\}\log 
\frac{1}{\delta} + d\right)\right),
\end{equation}
label requests are needed to construct $\widetilde{f}$ satisfying 
$R(\widetilde{f}) - R(f^*) \le \varepsilon$, with probability at least $1 - \delta$, for some 
distribution of the unlabeled data $P_X$. In particular, this result implies 
that the condition $\mathbf{s} < \infty$ is necessary for exponential 
savings in the number of label requests in this setup. Further, the aforementioned lower 
bound in Example \ref{ex:large_noise} shows that the bounded noise assumption 
\eqref{eq:massartsmargin} is necessary on the set where at least two functions in $\F$ 
disagree. Otherwise, one can easily choose a distribution on $\X \times \{0, 1\}$, so that the passive lower bound 
$\Omega\left(\frac{1}{\eps^2}\right)$ holds.

However, it is not immediately clear if the usual assumption $f^*_B \in \F$ is also needed 
when \eqref{eq:massartsmargin} holds for some $h > 0$.
The lower bound in \citep[Theorem 3]{Kaariainen05} exploits a specific situation where 
$f^*_B \notin 
\F$, and there exist $f, g \in \F$ that disagree on a set of \emph{infinite size} and this
leads to the agnostic lower bound 
$\Omega\left(\frac{1}{\eps^2}\right)$. To avoid this obstacle in passive learning with 
deterministic labeling, \citet{Bendavid14}\footnote{Their analysis is extended to the bounded noise case in \citep{bousquet2019fast}.} introduced the notion of the \emph{diameter} of $\F$. Recall that  
\[
D = \sup_{f, g \in \F}|\{x \in \X: f(x) \neq g(x)\}|.
\]
Similar to the star number, this complexity measure is infinite for many natural classes. 
However, it is still relevant, as many existing lower bounds in classification use the 
classes with a finite diameter \citep{Massart06, Audibert09}. 
Our second main result shows that if $D$ is finite, one can avoid the model 
misspecification problem in active learning under Massart's noise. 
\begin{Theorem}
	\label{thm:fin_dim} Assume that the diameter of $\F$ is equal to $D$ and the VC 
	dimension of $\F$ is equal to $d$. Fix $\eps, \delta \in (0,1]$. If Massart's noise 
	condition \eqref{eq:massartsmargin} is satisfied and 
	$h > 0$ is known, then there is an active learning algorithm (namely, Algorithm 
	\ref{alg:finitediam}) such that after requesting at most
	\[
	n = O\left(\frac{d\theta(\eps/h)}{h^2}\log^2 \left( \frac{d}{\varepsilon h\delta}\right) + 
	\frac{D}{h^2}\log\left(\frac{D}{\delta}\right)\right)
	\]
	labels, it returns a classifier $\widehat f$ satisfying, with probability at least $1 - 
	\delta$,
	\[
	R(\widehat f) - R(f^*) \leq \eps.
	\]
\end{Theorem}

This result implies that if both the diameter $D$ and the disagreement coefficient 
$\theta(\cdot)$ are bounded, then exponential savings are possible (without assuming 
$f^*_B \in \F$) if Massart's noise condition is satisfied with $h > 0$.
In \citep[Theorem 3]{Kaariainen05}, \citeauthor{Kaariainen05} proved that if the class $\F$ contains two functions $f_0$ and $f_1$ that agree on one point and disagree on infinitely many points (that is, $D = \infty$), then, for any active learning algorithm with uniformly bounded sample complexity for each $(\eps, \delta)$ there is a distribution $P_X$ and a deterministic labelling function $g \notin \F$ (i.e., $Y = g(X)$ almost surely) such that the algorithm needs
\begin{equation}
    \label{eq:kaariainen05_th3}
	\Omega\left( \frac{R(f^*)^2}{\eps^2} \log \frac1\delta \right)
\end{equation}
labels to produce a classifier with the excess risk at most $\eps$.
Further, by \citep[Theorem 10]{hanneke2015minimax} we have $\theta(\eps/h) \le \mathbf{s}$ for 
any distribution $P_X$ of the unlabeled data. Combining this result with Theorem 
\ref{thm:fin_dim}, \eqref{eq:activelowerbound}, and \eqref{eq:kaariainen05_th3}, we 
see that exponential savings are possible for all distributions satisfying  Massart's noise 
condition 
\eqref{eq:massartsmargin} with $h > 0$ \emph{if and only if} both the diameter $D$ and the star number $\mathbf{s}$ are finite. As we mentioned, both assumptions are quite restrictive and are not likely to be simultaneously satisfied for any non-trivial class of interest. However, if we are interested in distribution-dependent upper bounds, then by Theorem \ref{thm:fin_dim}, we only need that both the diameter and the disagreement coefficient are bounded to get exponential savings in the number of label requests.

The proof of Theorem \ref{thm:fin_dim} goes as follows. First, we fix $p = h/2$ and 
use
Algorithm \ref{alg_al_abs} to construct $\widehat f_{p}$. By Theorem 
\ref{thm:sample_comp_abs} we have, with probability at least $1 - \delta/3$, 
\[
R^{p}(\widehat f_{p}) - R(f^*) \le \eps.
\]
Observe that by the construction of Algorithm \ref{alg_al_abs}, and since we abstain only 
on the disagreement set of two classifiers, it holds that $|\{x \in \X: \widehat f_{p}(x) = 
*\}| \le D$. Therefore, if we specify the labels on these at most $D$ instances, we obtain 
a $\{0, 1\}$-valued classifier $\widehat{f}$. Since Massart's noise condition holds, a 
simple repeated-querying algorithm, similar to the one used in \citep[Theorem 
1]{Kaariainen05}, allows us to estimate the Bayes optimal rule $f^*_B$ on the 
finite set $\{x \in \X: \widehat f_{p}(x) = *\}$. The value $p = h/2$ is chosen 
to guarantee that
\begin{align*}
    &
    \Pr\left(f^*_B(X) \neq Y\ \text{and}\ \widehat f_{p}(X) = *\right) \le\left(1/2 - p\right)\Pr\left(\widehat f_{p}(X) = *\right),
\end{align*}
implying 
\[
R(\widehat f) - R(f^*) \le R^{p}(\widehat f_{p}) - R(f^*).
\]
The formal description of the algorithm of Theorem \ref{thm:fin_dim} is as follows.
\vskip+5pt
\begin{tcolorbox}[breakable, enhanced]
	\begin{myalgorithm}
		\label{alg:finitediam}
		\begin{enumerate}
			\item Fix $p = h/2$. Run Algorithm \ref{alg_al_abs} with the number of label 
			requests 
			sufficient to output $\widehat{f}_p$ satisfying, with probability at least $1 - 
			\delta/3$,
			\[
			R^p(\widehat{f}_p) - R(f^*) \le \eps/2.
			\]
			\item Set $\X_{\widehat{f}_p} = \{x \in \X: \widehat{f}_p(x) = *\}$.
			\item Sample $28D\log(6D/\delta) / (3h^2\eps)$ 
			fresh i.i.d. instances from $P_X$ and denote them by $Q$.
			\item Define $\widetilde{f}_D: \X_{\widehat{f}_p} \to \{0, 1\}$ as 
			follows: for each $x \in \X_{\widehat{f}_p}$ request the labels of all, but no 
			more than the first $2\log (6D / \delta) / h^2$ appearances of $x$ in the sample $Q$. Set $\widetilde{f}_D(x)$ to be equal to the majority vote of the labels of $x$ obtained this way with ties broken arbitrarily.  
			\item Set
			\[
			    \widehat{f}(x) =
    			\begin{cases}
    			\widetilde{f}_D(x), \quad \text{if } x \in \X_{\widehat{f}_p},\\
    			\widehat{f}_p(x), \quad \text{otherwise}.
    			\end{cases}
		    \]
			\item \textbf{return} $\widehat{f}$.
		\end{enumerate}
	\end{myalgorithm}
\end{tcolorbox}
\vskip+5pt
The full proof of 
Theorem \ref{thm:fin_dim} appears in Appendix \ref{sec:fin_dim_proof}. 
\begin{Remark}
The algorithm of Theorem \ref{thm:fin_dim} is improper. This means that the classifer
$\widehat f$ is not necessarily in $\F$. \citet[Corollary 13]{Bendavid14} show that in
passive learning with deterministic labeling, it is necessary to use improper learning 
algorithms to obtain the optimal sample complexity. A natural question is to understand 
if it is also the case in the context of Theorem \ref{thm:fin_dim}.
\end{Remark}

\section*{Acknowledgment}

The article was prepared within the framework of the HSE University Basic Research Program.  Nikita Puchkin is a Young Russian Mathematics award winner and would like to thank its sponsors and jury. Nikita Zhivotovskiy is funded in part by ETH Foundations of Data Science (ETH-FDS).

\bibliography{mybib}

\appendix

\section{Proof of Theorem \ref{thm:empirical_midpoint}}
\label{sec:empmidpointproof}

Throughout the proof, we identify $*$ with $1/2$ and convert $\{0, 1, *\}$-valued 
functions 
to $\{0, 1, 1/2\}$-valued ones by replacing $*$ with $1/2$ and vice versa. We refer to 
\citep{bousquet2019fast}, where the connections between
Chow's risk, strong convexity and the model selection aggregation are presented. In 
contrast, we provide a short and direct proof with explicit constants. As a result, our 
algorithm is simpler (see Remark \ref{rem:prevres}) and the proof is based only on the 
tools available in 
\citep{Vapnik74}.
In the notation of Theorem \ref{thm:empirical_midpoint} and Algorithm 
\ref{alg_al_abs_midp}, we need the following auxiliary lemma.

\begin{Lemma}
	\label{lem_mid-point_aux}
	With probability at least $1 - 4\delta / 7$, for all $f \in V$, we have
	\begin{align*}
	R^p(\widetilde f_p) - R(f)
	\leq 4 \alpha^2(n, \delta)
	&
	+ 8 \alpha(n, \delta)\D(V, L_2(P_n)) - p P_n(f - \widehat g)^2,
	\end{align*}
	where $\alpha(n, \delta)$ is given by \eqref{eq:alpha}.
\end{Lemma}

\begin{myproof}{}
	Lemma \ref{lem_mid-point_uniform} implies that, with probability at least $1 - 
	4\delta/7$, for any $f \in V$, it holds that
	\begin{align*}
	    R^p(\widetilde f_p) - R^p(f)
	    &
	    \leq R_n^p(\widetilde f_p) - R_n^p(f) + 4 \alpha^2(n, \delta)
	    + 8 \alpha(n, \delta) \sqrt{P_n(\widetilde f_p - f)^2}.
	\end{align*}
	By the definition of $\widetilde f_p$, there exists $f_0 \in V$ such that $\widetilde 
	f_p = (f_0 + \widehat g) / 2$.
	Then, by the convexity of the seminorm $L_2(P_n)$,
	\begin{align*}
	    \sqrt{P_n(\widetilde f_p - f)^2}
	    &
	    \leq \left(\sqrt{P_n(f_0 - f)^2}+\sqrt{P_n(\widehat g - f)^2} \right)/2
	    \leq \D(V, L_2(P_n)).
	\end{align*}
	This implies
	\begin{align*}
	    R^p(\widetilde f_p) - R^p(f)
	    &
	    \leq R_n^p(\widetilde f_p) - R_n^p(f)
	    + 4 \alpha^2(n, \delta)
	    + 8 \alpha(n, \delta)\D(V, L_2(P_n)). 
	\end{align*}
	Define the loss function corresponding to the risk \eqref{rp} as
	\begin{align}
    	\label{eq:lploss}
    	\ell^p(y, f(x))
    	&
    	= \Ind[y \neq f(x) \text{ and } f(x) \neq 1/2]
    	+ \left(\frac12 - p\right)\Ind[f(x) = 1/2].
	\end{align}
	A direct calculation shows that for any $\{0, 1\}$-valued functions $f, g$ and any $y 
	\in 
	\{0, 1\}$ it holds that for all $x \in \mathcal X$,
	\begin{align*}
    	\ell^p\left(y, \frac{f(x) + g(x)}2\right)
    	&
    	= \frac12 \ell^p(y, f(x)) + \frac12 \ell^p(y, g(x))
    	- p (f(x) - g(x))^2.
	\end{align*}
	Observe that $R^p_n(f) = P_n \ell^p(Y, f(X))$. By the definition of $\widetilde f_p$ 
	and the empirical risk minimizer $\widehat g$, we have for any $f \in V$,
	\begin{align*}
    	R^p_n(\widetilde f_p)
    	&
    	\leq R^p_n\left( \frac{f + \widehat g}2 \right)
    	\\&
    	= \frac12 R^p_n(f) + \frac12 R^p_n(\widehat g) - p P_n(f - \widehat g)^2
    	\\&
    	\leq R^p_n(f) - p P_n(f - \widehat g)^2,
	\end{align*}
	and the claim of the lemma follows.

\end{myproof}

\begin{myproof}{of Theorem \ref{thm:empirical_midpoint}}
	To prove \eqref{eq:datadepbound}, take $h \in V$, such that $P_n(h - \widehat g)^2 
	\geq  \D^2(V, L_2(P_n))/4$.
	The existence of such an $h$ follows from the definition of $V$.
	Due to Lemma \ref{lem_vc_uniform}, there is an event $E$ such that $\Pr(E) \ge 1 - 
	3\delta / 7$ and $f^*$ belongs to $V$ on $E$.
	Furthermore, on this event, it holds that
	\begin{align*}
    	R(h) - R(f^*)
    	&
    	\leq R_n(h) - R_n(f^*) + 2\alpha^2(n, \delta)
    	\\&\quad
    	+ 2\alpha(n, \delta) \sqrt{P_n(h - f^*)^2}
    	\\&
    	\leq R_n(h) - R_n(f^*) + 2\alpha^2(n, \delta)
    	\\&\quad
    	+ 2\alpha(n, \delta) \D(V, L_2(P_n)).
	\end{align*}
	Since $R_n(f^*) \geq R_n(\widehat g)$, we have
	\begin{align*}
	R_n(h) - R_n(f^*)
	&
	\leq R_n(h) - R_n(\widehat g)
	\\&
	\leq 2\alpha^2(n, \delta) + 2\alpha(n, \delta) \sqrt{P_n(h - \widehat g)^2}
	\\&
	\leq 2\alpha^2(n, \delta) + 2\alpha(n, \delta) \D(V, L_2(P_n)),
	\end{align*}
	where the second and third inequalities hold since $h, \widehat g \in V$.
	This yields
	\[
	R(h) - R(f^*)
	\leq 4\alpha^2(n, \delta) + 4\alpha(n, \delta) \D(V, L_2(P_n)).
	\]
	Applying Lemma \ref{lem_mid-point_aux} and the union bound, we have, with 
	probability at least $1 - \delta$,
	\begin{align*}
	R^p(\widetilde f_p) - R(f^*)
	&
	= R^p(\widetilde f_p) - R(h) + R(h) - R(f^*)
	\\&
	= R^p(\widetilde f_p) - R^p(h) + R(h) - R(f^*)
	\\&
	\leq 8 \alpha^2(n, \delta) + 12 \alpha(n, \delta)\D(V, L_2(P_n))
	- p P_n(h - \widehat g)^2
	\\&
	\leq 8 \alpha^2(n, \delta) + 12 \alpha(n, \delta)\D(V, L_2(P_n))
	- \frac p4 \D^2(V, L_2(P_n)).
	\end{align*}
	Hence, the proof of \eqref{eq:datadepbound} is finished.
	To prove \eqref{eq:belowzero}, we consider the largest root $x_+$ of the equation
	\[
		\frac p4 x^2 - 12\alpha(n, \delta) x - 8 \alpha^2(n, \delta) = 0
	\]
	and show that it is smaller than $49\alpha(n, \delta)/p$.
	Indeed, taking into account that $p \leq 1/2$, we obtain
	\begin{align*}
		x_+
		&
		=
		\frac{2}{p} \left( 12\alpha(n, \delta) + \sqrt{ 144 \alpha^2(n, \delta) + 8p \alpha^2(n, \delta)} \right)
		\\&
		\leq \frac{2 \alpha(n, \delta)}{p} \left( 12 + \sqrt{148} \right)
		< \frac{49\alpha(n, \delta)}p.
	\end{align*}
	Since $\frac p4 x^2 - 12\alpha(n, \delta) x - 8 \alpha^2(n, \delta) < 0$ for all $x > x_+$, it holds that $R^p(\widetilde f_p) - R(f^*) < 0$ whenever $\D(V, L_2(P_n)) \geq 49 \alpha(n, \delta) / p$.

\end{myproof}

\section{Proof of Theorem \ref{thm:sample_comp_abs}}
\label{sec:sample_comp_abs_proof}
As we mentioned, our proof follows the standard arguments with several technical 
modifications needed to incorporate the result of Theorem \ref{thm:empirical_midpoint}.
For the ease of exposure, we split the proof into several steps.

\medskip

\noindent
{\bf Step 0.}
\quad 
Note that, for any $f, g \in V_{j - 1}$, it holds that
\[
|S_j| \left( R_{S_j}(f) - R_{S_j}(g) \right) 
= n_j \left( R_{Q_j}(f) - R_{Q_j}(g) \right),
\]
because $f(x) = g(x)$ for all $x \in Q_j \backslash D_j$.

\medskip

\noindent
{\bf Step 1.}
\quad
Let $E_1$ be an event such that $\Pr(E_1) \geq 1 - \delta_1$, $f^*$ belongs to $V_1$ on 
$E_1$, and, moreover,
\begin{align*}
    R^p(\widetilde f_1) - R(f^*)
    \leq 8 \alpha^2(n_1, \delta_1)
    &
    + 12 \alpha(n_1, \delta_1) \D(V_1, L_2(P_n))
    \\&
    - \frac p4 \D^2(V_1, L_2(P_n)),
\end{align*}
where $\widetilde f_1$ is the output of Algorithm \ref{alg_al_abs_midp} applied to the 
class $V_0 = \F$ with the confidence $\delta_1$.
The existence of such $E_1$ is guaranteed by Theorem \ref{thm:empirical_midpoint}.
Given an integer $j \geq 2$, define an event $E_j$ as follows. Let $E_j$ be such that $f^* 
\in V_j$ on $E_j$, and, on the same event, it holds that
\begin{align*}
    R^p(\widetilde f_j) - R(f^*)
    \leq 8 \alpha^2(n_j, \delta_j)
    &
    + 12 \alpha(n_j, \delta_j) \D(V_j, L_2(P_n))
    \\&
    - \frac p4 \D^2(V_j, L_2(P_n)),
\end{align*}
where $\widetilde f_j$ is the output of Algorithm \ref{alg_al_abs_midp} applied to the 
class $V_{j-1}$ with the confidence $\delta_j$.
Theorem \ref{thm:empirical_midpoint} implies that $\Pr(E_j \,\vert\, E_1, \dots, E_{j-1}) 
\geq 1 - \delta_j$.
Note that, by the definition of $E_1, \dots, E_J$,
\begin{align*}
    \Pr\left( \bigcap\limits_{j=1}^J E_j \right)
    &
    = \Pr\left( E_J \,\vert\, E_1, \dots, E_{J-1} \right)
    \cdot \Pr\left( E_{J-1} \,\vert\, E_1, \dots, E_{J-2} \right) \cdot
    \dots \cdot \Pr\left( E_2 \,\vert\, E_1 \right) \Pr\left( E_1 \right) 
    \\&
    \geq \prod\limits_{j=1}^J (1 - \delta_j)
    \geq 1 - \sum\limits_{j=1}^J \delta_j
    \\&
    = 1 - \sum\limits_{j=1}^J \frac{\delta}{(1+j)^2}
    \geq 1 - \sum\limits_{j=1}^\infty \frac{\delta}{(1+j)^2}
    \\&
    \geq 1 - \frac{2\delta}3.
\end{align*}
In particular, this yields that, with probability at least $1 - 2\delta/3$, $f^* \in V_j$ for any 
$j \in \{1, \dots, J\}$.

\medskip

\noindent
{\bf Step 2.}
\quad Consider an event $E_{\cap} = \cap_{j=1}^J E_j$, $\Pr(E_{\cap}) \geq 1 - 2\delta/3$, 
where $E_1, \dots, E_J$ were introduced in the previous step.
We prove that on this event 
\[
R^p(\widehat f_p) - R(f^*) \leq \eps.
\]
Theorem \ref{thm:empirical_midpoint} 
implies that, for any $j \in \{1, \dots, J\}$, one has either
$\D(V_j, L_2(P_n)) < 49 \alpha(n_j, \delta_j) / p$ or $R^p(\widetilde f_j) - R(f^*) < 0$ on 
$E_{\cap}$.
Hence, if the procedure terminates ahead of time, we have 
$R^p(\widehat f_p) - R(f^*) < 0$, with probability at least $1 - 2 \delta/3$.
In this case the proof is complete.
Otherwise, we have $\D(V_j, L_2(P_{Q_j})) < 49 \alpha(n_j, \delta_j) / p$ for all $j \in \{1, 
\dots, J\}$, with probability at least $1 - 2 \delta/3$.
Then, on the final iteration, we obtain
\begin{align*}
    R^p(\widetilde f_J) - R(f^*)
    \leq 8 \alpha^2(n, \delta)
    &
    + 12 \alpha(n, \delta) \D(V_J, L_2(P_n))
    - \frac p4 \D^2(V_J, L_2(P_n)).
\end{align*}
Maximizing the right-hand side over $\D(V_J, L_2(P_n))$ and taking into account that $p \leq 1/2$, we get
\begin{align*}
    R^p(\widetilde f_J) - R(f^*)
    &
    \leq 8 \alpha^2(n_J, \delta_J) + \frac{144 \alpha^2(n_J, \delta_J)}p
    \\&
    \leq \frac{148 \alpha^2(n_J, \delta_J)}{p},
\end{align*}
with probability at least $1 - 2\delta/3$.
Recall that $n_J = 2^{J - 1}$ and $\delta_J = \frac{\delta}{(J + 1)^2}$. Since $J$ satisfies 
the condition
\[
\frac{148 \alpha^2(n_J, \delta_J)}{p} \leq \eps,
\]
we have $R^p(\widehat f_p) - R(f^*) = R^p(\widetilde f_J) - R(f^*) \leq \eps$, with 
probability at least $1 - 2\delta/3$.

\medskip

\noindent
{\bf Step 3.}
\quad
The total number of labels requested by Algorithm \ref{alg_al_abs} is equal to
\[
\sum\limits_{j=1}^{T \wedge J} \sum\limits_{i=1}^{n_j} \Ind\left( X_i \in 
\dis(V_{j-1}) \right),
\]
where $T$ is the iteration when the procedure terminates.

Previously, we proved that either $\D(V_{j-1}, L_2(P_n)) < 49 \alpha(n_{j-1}, \delta_{j-1}) / 
p$ 
for all 
$j \leq J$ or we terminate at the moment $j < J$ and $R^p(\widetilde f_j) - R(f^*) < 0 < 
\eps$, with probability at least $1 - 2\delta/3$.
Therefore, we may assume in the analysis that $\D(V_{j-1}, L_2(P_n)) < 49 \alpha(n_{j-1}, 
\delta_{j-1}) / p$ 
for all $j \leq T$.
Applying Lemma \ref{lem_vc_uniform}, we have, with probability at least $1 - 3\delta_j/7$,
\begin{align*}
    \sup\limits_{f, g \in V_j} P|f - g|
    &
    \leq \sup\limits_{f, g \in V_j} \big( P_{Q_j} |f - g|
    \\&\qquad
    + \alpha(n_j, \delta_j) \sqrt{P_{Q_j} |f - g| } 
    + \alpha^2(n_j, \delta_j) \big)
    \\&
    \leq \D^2(V_j, L_2(P_{Q_j})) + \alpha(n_j, \delta_j) \D(V_j, L_2(P_{Q_j}))
    \\&\qquad
    + \alpha^2(n_j, \delta_j)
    \\&
    \le \frac{50^2 \alpha^2(n_j, \delta_j)}{p^2}.
\end{align*}
Consequently, on this event, $\dis(V_{j-1}) \subseteq \dis(\mathcal B(f^*, \xi_{j-1}))$, 
where we defined 
\[
\xi_{j-1} = \frac{50^2 \alpha^2(n_{j-1}, \delta_{j-1})}{p^2},
\]
and
\[
\mathcal B(f^*, r) = \left\{ f \in \F : P|f - f^*| \leq r \right\}.
\]
This yields $P_X\left( \dis(V_{j-1}) \right) \leq \theta(\xi_{j-1}) \xi_{j-1}$.
Due to Bernstein's inequality and since for $a, b \geq 0$,  $\sqrt{2ab} \leq a/2 + b$, 
(conditionally on $V_{j-1}$) it holds that, with probability at least $1 - \delta_j / 56$,
\begin{align*}
&
\sum\limits_{i=1}^{n_j} \Ind\left( X_i \in \dis(V_{j-1}) \right)
\\&
\leq n_j \theta(\xi_{j-1}) \xi_{j-1}
+ \sqrt{2n_j \theta(\xi_{j-1})\xi_{j-1} \log(56/\delta_j)}
\\&\qquad
+ 2\log(56/\delta_j) 
\\&
\leq \frac 32 n_j \theta(\xi_{j-1}) \xi_{j-1} + 3\log(56/\delta_j)
\\&
= 3n_{j-1} \theta(\xi_{j-1}) \xi_{j-1} + 3\log(56/\delta_j)
\\&
\le \frac{4 \cdot 50^2 \theta(\xi_{j-1})}{p^2} \left(9d + 3d j + 2\log(1 + j) + \log(56/\delta) 
\right)
\\&\qquad
+ 6 \log(1 + j) + 3\log(56/\delta)
\\&
\le \frac{4 \cdot 50^2 \theta(\xi_{j-1})}{p^2} \left(9d + (3d + 3) j + 2\log(56/\delta) \right).
\end{align*}
By the union bound, with probability at least
\begin{align*}
    \Pr(E_\cap) - \sum\limits_{j=1}^J \frac{3\delta_j}7 - \sum\limits_{j=1}^J \frac{\delta_j}{56}
    &
    \geq 1 - \frac{2\delta}3 - \frac{2\delta}3\left(\frac37 + \frac1{56} \right)
    \\&
    \ge 1 - \delta,
\end{align*}
the total number of requested labels is not greater than
\begin{align*}
    &
    \sum\limits_{j=1}^{T \wedge J} \frac{4 \cdot 50^2 \theta(\xi_{j-1})}{p^2} \left(9d + (3d + 2) j + \log(56/\delta) \right)
    \\&\quad
    + \sum\limits_{j=1}^{T \wedge J} \left( 6 \log(1 + j) + 3\log(56/\delta) \right),
\end{align*}
and, using the fact that $\theta(\cdot)$ is a decreasing function, we obtain
\begin{align}
\label{eq:sharpestbound}
&\notag
\sum\limits_{j=1}^{T \wedge J} \frac{4 \cdot 50^2 \theta(\xi_{j-1})}{p^2} \left(9d + (3d 
+ 3) j + 2\log(56/\delta) \right)
\\&\notag
\leq \sum\limits_{j=1}^{T \wedge J} \frac{4 \cdot 50^2 \theta(\xi_{j-1})}{p^2} \left(9d + 
(3d + 3) j + 2\log(56/\delta) \right)
\\&\notag
\leq \frac{36 \cdot 50^2 \theta(\xi_{J-1}) d J}{p^2} + \frac{2 \cdot 50^2 \theta(\xi_{J-1}) 
(3d + 3) J(J+1)}{p^2}
\\&\quad
+ \frac{8 \cdot 50^2 \theta(\xi_{J-1}) J \log(56/\delta)}{p^2} 
\\&\notag
\lesssim \frac{\theta(\eps/p)}{p^2} \left( d \log^2\left(\frac{d}{p\eps} \right) + 
\log\left(\frac{d}{p\eps} \right) \log\left(\frac{1}{\delta}\right) \right)
\\&\quad\notag
+ \frac{\theta(\eps/p)}{p^2} \log\left(\frac{1}{\delta}\right) \log\log\left(\frac{1}{\delta}\right).
\end{align}
To prove the last inequality, we took into account that, by the definition of $J$,
\[
\xi_{J-1} = 
50^2\alpha^2(n_{J-1}, \delta_{J-1}) / p^2 \geq (50^2 \eps) / (148 p) > \eps/p.
\]
Moreover, it is easy to see that
\begin{align*}
    n_{J-1}
    &
    \lesssim \frac{d \log(d/\eps)}{p\eps} + \log\frac1{\delta_{J-1}}
    \lesssim \max\left\{\frac{d \log(d/\eps)}{p\eps}, \log\frac1{\delta_{J-1}}\right\},
\end{align*}
which yields $J \lesssim \log\left(\frac{d}{p\eps} \right) + \log\log(1/\delta)$. Therefore, 
the upper bound \eqref{eq:sharpestbound} follows. For the sake of presentation, in our 
statement we use a simple relaxation of \eqref{eq:sharpestbound}. 

\hfill$\square$

\section{Proof of Proposition \ref{prop:adaptivity} and Proposition 
\ref{prop:secprop}}
\begin{myproof}{Proof of Proposition \ref{prop:adaptivity}}
	Let $\eta(x) = \Pr(Y = 1 | X = x)$.
	Theorem \ref{thm:sample_comp_abs} yields that $R^{p}(\widehat f_p) - R(f^*_B) 
	\leq \eps$, with probability at least $1 - \delta$.
	By \citep[an equality on page 341]{Boucheron05b}, we have
	\begin{align*}
	\eps
	&
	\geq  R^{p}(\widehat f_p) - R(f^*_B)
	\\&
	= \E |2\eta(X) - 1| \Ind\left[ \widehat f_p(X) \neq f^*_B(X) \text{ and } \widehat f_p(X) \neq * \right]
	\\&\qquad
	+ \E \left( \frac12 - p - \min\{\eta(X), 1 - \eta(X)\} \right) \Ind\left[ 
	\widehat f_p(X) = * \right]
	\\&
	\geq 0 + \left(\frac h2 - p\right) \Pr\left( \widehat f_p(X) = * \right)
	\\&
	\geq \frac h4 \Pr\left( \widehat f_p(X) = * \right),
	\end{align*}
	where we used the noise condition \eqref{eq:massartsmargin}.
	To finish the proof, note that on the same event we have
	\begin{align*}
    	R^0(\widehat f_p) - R(f^*)
    	&
    	= R^p(\widehat f_p) - R(f^*) + p \Pr(\widehat f_p(X) = *)
    	\\&
    	\leq \eps + \frac{4p\eps}h \leq 2\eps.
	\end{align*}
	The claim follows.
	
\end{myproof}

\begin{myproof}{Proof of Proposition \ref{prop:secprop}}
	As above, let $\eta(x) = \Pr(Y = 1 | X = x)$. We show a slightly more general result, 
	that is, for 
	any $u > 0$, with probability at least $1 - \delta$,
	\begin{align*}
    	&
    	\Pr\left( \widetilde f_p(X) = * \text{ and } |2\eta(X) - 1| \geq 2(p + u)  \right)
    	\leq \frac{592}{npu} \left( 3d \log \frac{e(2n \vee d)}d + \log \frac{56}\delta \right).
    \end{align*}
	First, maximizing \eqref{eq:datadepbound} with respect to $\D(V, L_2(P_n))$, we 
	obtain that
	\begin{align*}
    	R^p(\widetilde f_p) - R(f_B^*)
    	&
    	\leq \frac{148 \alpha^2(n, \delta)}{p}
    	= \frac{592}{np} \left( 3d \log \frac{e(2n \vee d)}d + \log \frac{56}\delta \right),
    \end{align*}
	with probability at least $1 - \delta$.
	On the other hand, similarly to the proof of Proposition \ref{prop:adaptivity}, we have
	\begin{align*}
	&
	R^{p}(\widehat f_p) - R(f^*_B)
	\\&
	= \E |2\eta(X) - 1| \Ind\left[ \widetilde f_p(X) \neq f^*_B(X) \text{ and } \widetilde 
	f_p(X) \neq * \right]
	\\&\qquad
	+ \E \left( \frac12 - p - \min\{\eta(X), 1 - \eta(X)\} \right) \Ind\left[ 
	\widetilde f_p(X) = * \right]
	\\&
	\geq 0 + \E \left( \frac12 - p - \min\{\eta(X), 1 - \eta(X)\} \right)
	\Ind\left[ \widetilde f_p(X) = * \text{ and } |2\eta(X) - 1| \geq 2(p + u)  \right]
	\\&
	\geq u \Pr\left( \widetilde f_p(X) = * \text{ and } |2\eta(X) - 1| \geq 2(p + u)  \right).
	\end{align*}
	Combining these two bounds, we obtain
	\begin{align*}
    	&
    	\Pr\left( \widetilde f_p(X) = * \text{ and } |2\eta(X) - 1| \geq 2(p + u)  \right)
    	\leq \frac{592}{npu} \left( 3d \log \frac{e(2n \vee d)}d + \log \frac{56}\delta \right).
	\end{align*}
	The claim follows by choosing $u = p$.

\end{myproof}

\section{Proof of Theorem \ref{thm:fin_dim}}
\label{sec:fin_dim_proof}

{\bf Step 1.}
First, we prove that conditionally on the observations required to construct 
$\widehat{f}_p$, with probability at least $1 - \delta/3$, for any $x \in \X_{\widehat{f}_p}$ 
it holds that either 
\begin{equation}
\label{eq:condition}
f^*_B(x) = \widetilde{f}_D(x),\ \text{or}\ \Pr(\{x\}) < \frac{\eps}{2D}.
\end{equation}
Assume that there is at least one $x \in \X_{\widehat{f}_p}$ such that $\Pr(\{x\}) \ge \eps / (2D)$, since otherwise we already have 
\eqref{eq:condition}. For some integer $m$, we estimate the number of unlabeled 
instances sufficient to observe each $x \in \X_{\widehat{f}_p}$ having 
$\Pr(\{x\}) \ge \eps / (2D)$ at least $m$ times.
Fix any such $x$ and apply the Bernstein inequality to Bernoulli random variables $\eta_j 
= \Ind[x \text{ occurred on the $j$-th trial}]$, $1 \leq j \leq N$.
Then the probability that a point $x \in \X_{\widehat{f}_p}$ with probability mass 
$\Pr(\{x\}) \ge \eps / (2D)$ was observed less than $m$ times after $N \geq m \Pr(\{x\})$ trials does not exceed
\begin{align*}
    &
    \exp\left( -\frac{N(\Pr(\{x\}) - m/N)^2}{2 \Pr(\{x\}) (1 - \Pr(\{x\})) + 2(\Pr(\{x\}) - m/N) / 3}\right)
    \\&
    \leq \exp \left( -\frac{3N(\Pr(\{x\}) - m/N)}{8} \right)
    \\&
    \leq \exp \left( -\frac{3N\eps - 6m D}{16 D} \right).
\end{align*}
Since the diameter is finite we have 
by the construction of Algorithm \ref{alg_al_abs} that $|\X_{\widehat{f}_p}| \le D$.
Thus, by the union bound, the probability that there exists $x \in \X_{\widehat{f}_p}$, 
$\Pr(\{x\}) \ge \eps / (2D)$ such that it occurred less than $m$ times after $N$ trials 
is not greater than
\[
    D \exp \left( -\frac{3N\eps - 6m D}{16 D} \right).
\]
Thus, we need at most $D \left(6m + 16 \log(3D / \delta)\right) / (3\eps)$ unlabeled instances to satisfy this, 
with probability at least $1 - \delta/3$. 

Recall that $f^*_B(x) = \Ind[\Pr(Y=1|X = x)\ge 1/2]$. By our assumption, we have for 
any $x$, $|2\Pr(Y=1|X = x)-1| \ge h$. By Hoeffding's inequality and the union bound, the 
probability that there is $x \in \X_{\widehat{f}_p}$ such that the majority vote 
$\widetilde{f}_D(x)$ is not equal to $f^*_B(x)$, is bounded by
\[
2D\exp\left(-2m(h/2)^2\right) \le \delta/3,
\]
whenever the number of label requests $m$ for each instance satisfies $m \ge 2\log(6D/\delta) / h^2$. This proves \eqref{eq:condition}. Since $|2\Pr(Y=1|X)-1| \ge h$ almost surely, we have 
\begin{align*}
    &
    \Pr(f^*_B(X) \neq Y| X= x)
    \\&
    = \Pr(Y=1|X=x) \wedge (1 - \Pr(Y=1|X=x))
    \\&
    \le (1 - h)/2.
\end{align*}
This implies that on the event where \eqref{eq:condition} holds, since 
$|\X_{\widehat{f}_p}| \le D$, we also have
\begin{align}
    \label{eq:riskofmajvote}
    &\notag
    \Pr(\widetilde{f}_D(X) \neq Y\ \text{and}\ X \in \X_{\widehat{f}_p})
    \\&
    \le \left((1 - h)P_X\left(\X_{\widehat{f}_p}\right) + \eps\right)/2.
\end{align}

{\bf Step 2.} By the union bound, \eqref{eq:riskofmajvote} and the first step of our 
algorithm we have, with probability at least $1 - \delta$, 
\begin{align*}
    R(\widehat{f})
    &
    = \Pr(\widehat{f}_p(X) \neq Y\ \text{and}\ X \notin \X_{\widehat{f}_p})
    + \Pr(\widetilde{f}_D(X) \neq Y\ \text{and}\ X \in \X_{\widehat{f}_p})
    \\&
    \le \Pr(\widehat{f}_p(X) \neq Y\ \text{and}\ X \notin \X_{\widehat{f}_p})
    + \left((1 - h)P_X\left(\X_{\widehat{f}_p}\right) + \eps\right)/2
    \\&
    = R^p(\widehat{f}_p) + \eps/2
    \\& \le R(f^*) + \eps.
\end{align*}
Thus, the desired risk bound follows.

{\bf Step 3.} It is only left to estimate the number of label requests. Using the sample 
complexity bound of Theorem \ref{thm:sample_comp_abs} 
and that for each $x 
\in \X_{\widehat{f}_p}$ we request at most
\[
    \left\lceil\frac{2\log(6D/\delta)}{h^2}\right\rceil
\]
labels, we have that the total number of label requests is 
\[
n = O\left(\frac{d\theta(\eps/h)}{h^2}\log^2 \left( \frac{d}{\varepsilon h\delta}\right) + 
\frac{D}{h^2}\log\left(\frac{D}{\delta}\right)\right).
\] 
The claim follows.

\hfill$\square$

\section{Auxiliary results}
The next result is due to \citet[Theorem 12.2]{Vapnik74} presented in the form of 
\citet[Theorem 5.1]{Boucheron05b}. 
\begin{Lemma}
	\label{lem_vc74}
	Let $\F$ be a class of $\{0,1\}$-valued functions and for $\delta \in (0, 
	1)$, introduce
	\[
	\sigma^2(n, \delta) = \frac4n \left( \log \mathcal S_\F(2n) + \log \frac8\delta 
	\right).
	\]
	Then, with probability at least $1 - \delta$, for all $f \in \F$, it holds that
	\[
	P_n f - P f \leq \min\left\{ \sigma^2(n, \delta) + \sigma(n, \delta)\sqrt{P f}, \sigma(n, 
	\delta) \sqrt{P_n f} \right\}
	\]
	and
	\[
	P f - P_n f \leq \min\left\{\sigma^2(n, \delta) + \sigma(n, \delta)\sqrt{P_n f}, 
	\sigma(n, \delta) \sqrt{P f} \right\}. 
	\]
\end{Lemma}

\begin{Lemma}
	\label{lem_vc_uniform}
	Let $\F$ be a class of $\{0,1\}$-valued functions with VC dimension $d$. Fix $\delta 
	\in (0, 1)$.
	Let
	\[
	\beta^2(n, \delta) = \frac4n \left( 2d \log \frac{e(2n\vee d)}d + \log \frac{24}\delta 
	\right).
	\]
	Then, the following inequalities hold simultaneously, with probability at least $1 - 
	\delta$ for all $f, g \in \F$:
	\begin{align}
	\label{vc_uniform_1}
	&\notag
	\left| R(f) - R(g) - R_n(f) + R_n(g) \right|
	\\&
	\leq 2\beta^2(n, \delta)
	+ 2\beta(n, \delta) \sqrt{P_n |f - g| \wedge P |f - g|},
	\end{align}
	and
	\begin{align}
	\label{vc_uniform_2}
	&
	\left| P |f - g| - P_n |f - g| \right|
	\leq \beta^2(n, \delta) + \beta(n, \delta) \sqrt{P_n |f - g| \wedge P |f - g|}.
	\end{align}
\end{Lemma}

Lemma \ref{lem_vc_uniform} is a simple corollary of Lemma \ref{lem_vc74}. 
Similar bounds are used in the proofs in \citep{dasgupta2008general, Hsu10, 
zhang2014beyond, 
	bousquet2019fast}.
We provide the proof of Lemma \ref{lem_vc_uniform} for the sake of completeness. We 
remark that in our analysis, the logarithmic factors can be improved. In particular, the 
results in \citep{gine2006concentration, Zhivotovskiy18} allow refining the logarithmic 
factors in Lemma \ref{lem_vc_uniform}. Moreover, the techniques in \citep[Theorem 
4]{hanneke2015minimax} give a better joint dependence on the VC dimension $d$ and 
the star number $\mathbf{s}$ in a similar context. An adaptation of these techniques is a 
natural direction of future research.

\begin{myproof}{of Lemma \ref{lem_vc_uniform}}
	By Sauer's lemma \cite[Theorem 1]{Sauer72}, we have
	\[
    	\log S_\F(2n)
    	\leq d \log \frac{e(2n \vee d)}d.
	\]
	Then, the inequality \eqref{vc_uniform_2} follows from Lemma \ref{lem_vc74} 
	applied to the class $\F\Delta\F = \{ |f - g| : f, g \in \F\}$ with the confidence 
	$\delta/3$ and the fact that
	\[
	    \log S_{\F\Delta\F}(2n) \leq 2\log S_\F(2n).
	\]
	To prove \eqref{vc_uniform_1}, we rewrite $\Ind[f(x) \neq y] - \Ind[g(x) \neq y]$ in the form
	\begin{align*}
    	\Ind[f(x) \neq y] - \Ind[g(x) \neq y]
    	&
    	= \Ind[f(x) \neq y \text{ and } g(x) = y]
    	\\&\quad
    	- \Ind[g(x) \neq y \text{ and } f(x) = y]
	\end{align*}
	and apply Lemma \ref{lem_vc74} to the classes $\F_1 = \{(x, y) \mapsto \Ind[f(x) \neq 
	y \text{ and } g(x) = y] : f, g \in \F \}$ and $\F_2 = \{(x, y) \mapsto \Ind[f(x) = y \text{ 
	and } g(x) \neq y] : f, g \in \F \}$.
	It is easy to see that
	\[
	\log S_{\F_1}(2n) = \log S_{\F_2}(2n) \leq 2\log S_\F(2n).
	\]
	Then, with probability at least $1 - \delta/3$, it holds that
	\begin{align*}
	    &
	    \big| P \Ind[f(X) \neq Y \text{ and } g(X) = Y]
	    - P_n \Ind[f(X) \neq Y \text{ and } g(X) = Y] \big|
	    \\&
	    \leq \beta^2(n, \delta)
	    + \min\big\{ \beta(n, \delta) \sqrt{P \Ind[f(X) \neq Y \text{ and } g(X) = Y]}, 
	    \\&\qquad\qquad\quad
	    \beta(n, \delta) \sqrt{P_n \Ind[f(X) \neq Y \text{ and } g(X) = Y]} \big\}
	    \\&
	    \leq \beta^2(n, \delta) + \beta(n, \delta) \sqrt{P(f - g)^2 \wedge P_n(f - g)^2},
	\end{align*}
	and, with the same probability,
	\begin{align*}
	    &
	    \big| P \Ind[g(X) \neq Y \text{ and } f(X) = Y]
	    - P_n \Ind[g(X) \neq Y \text{ and } f(X) = Y] \big|
	    \\&
	    \leq \beta^2(n, \delta)
	    + \min\big\{ \beta(n, \delta) \sqrt{P \Ind[g(X) \neq Y \text{ and } f(X) = Y]},
	    \\&\qquad\qquad\quad
	    \beta(n, \delta) \sqrt{P_n \Ind[g(X) \neq Y \text{ and } f(X) = Y]} \big\}
	    \\&
	    \leq \beta^2(n, \delta) + \beta(n, \delta) \sqrt{P(f - g)^2 \wedge P_n(f - g)^2}.
	\end{align*}
	Then, with probability at least $1 - 2\delta/3$, we have
	\begin{align*}
	    &
	    \left| R(f) - R(g) - R_n(f) + R_n(g) \right|
	    \\&
	    \leq 2\beta^2(n, \delta) + 2\beta(n, \delta) \sqrt{P(f - g)^2 \wedge P_n(f - g)^2}.
	\end{align*}
	Finally, the union bound concludes the proof. 

\end{myproof}

Our next result provides a similar uniform bound for Chow's risk.

\begin{Lemma}
	\label{lem_mid-point_uniform}
	Let $\F$ be a class of $\{0, 1\}$-valued functions with VC-dimension $d$ and let 
	\[
	\G = \frac{\F + \F}2
	= \left\{ \frac{f_1 + f_2}2 : f_1, f_2 \in \F \right\}.
	\]
	Denote
	\[
	\gamma^2(n, \delta) = \frac4n \left( 3d \log \frac{e(2n \vee d)}d + \log \frac{32}\delta 
	\right).
	\]
	Then, with probability at least $1 - \delta$, for all $f \in \F$, $g \in \G$, it 
	holds that
	\begin{align*}
	    &
	    \left| R^p(f) - R^p(g) - R^p_n(f) + R^p_n(g) \right|
	    \\&
	    \leq 4 \gamma^2(n, \delta) + 8 \gamma(n, \delta) \sqrt{ P_n(f - g)^2 \wedge P(f - g)^2 }.
	\end{align*}
\end{Lemma}

\begin{myproof}{}
	Recall the definition \eqref{eq:lploss} of the $\ell^p$ loss.
	We have $R^p(f) = P\ell^p(Y, f(X))$ and $R^p_n(f) = P_n \ell^p(Y, f(X))$. For any $f \in 
	\F$, $g \in \G$, and for any $y \in \{0, 1\}$, $x \in \X$, it holds 
	that
	\begin{align*}
	&
	\ell^p(y, g(x)) - \ell^p(y, f(x))
	\\&
	= \Ind[g(x) = y \text{ and } f(x) \neq y]
	\\&\quad
	- \Ind[g(x) = 1 - y \text{ and } f(x) = y]
	\\&\quad
	+ \left(\frac12 - p\right) \Ind\left[ g(x) = \frac12 \text{ and } f(x) = y \right]
	\\&\quad
	- \left(\frac12 + p\right) \Ind\left[ g(x) = \frac12 \text{ and } f(x) \neq y \right].
	\end{align*}
	Apply Lemma \ref{lem_vc74} to each term in the right hand side.
	First, Sauer-Shelah lemma yields
	\[
	\log \mathcal \mathcal S_{\G}(2n)
	\leq 2 \log \mathcal S_\F(2n)
	\leq 2d \log \frac{e(2n \vee d)}d.
	\]
	Define the classes 
	\begin{align*}
	\G_1 &= \{(x, y) \mapsto \Ind[g(x) = y \text{ and } f(x) \neq y] : g \in \G, f \in \F\},
	\\
	\G_2 &= \{(x, y) \mapsto \Ind[g(x) = 1 - y \text{ and } f(x) = y] : g \in \G, f \in \F\},
	\\
	\G_3 &= \{(x, y) \mapsto \Ind[ g(x) = 1/2 \text{ and } f(x) = y ] : g \in \G, f \in \F\},
	\\
	\G_4 &= \{(x, y) \mapsto \Ind[ g(x) = 1/2 \text{ and } f(x) \neq y ] : g \in \G, f \in \F\}.
	\end{align*}
	It holds that
	\begin{align*}
	    \max\limits_{1 \leq k \leq 4} \log S_{\G_k}(2n)
	    &
	    \leq \log \mathcal S_{\G}(2n) + \log \mathcal S_{\F}(2n)
	    \\&
	    \leq 3d \log \frac{e(2n \vee d)}d.
	\end{align*}
	Using $
	\Ind[g(x) = y \text{ and } f(x) \neq y] \leq \Ind[g(x) \neq f(x)]
	\leq 4(f(x) - g(x))^2,$
	we deduce that, with probability at least $1 - \delta/4$, it holds that
	\begin{align*}
	&
	\big| P \Ind[g(X) = Y \text{ and } f(X) \neq Y]
	- P_n \Ind[g(X) = Y \text{ and } f(X) \neq Y] \big|
	\\&
	\leq \gamma^2(n, \delta) + 2\gamma(n, \delta) \sqrt{ P_n(f - g)^2 \wedge P(f - g)^2}.
	\end{align*}
	Similarly, we have, with probability at least $1 - \delta/4$,
	\begin{align*}
	&
	\big| P \Ind[g(x) = 1 - y \text{ and } f(x) = y]
	- P_n \Ind[g(x) = 1 - y \text{ and } f(x) = y] \big|
	\\&
	\leq \gamma^2(n, \delta) + 2\gamma(n, \delta) \sqrt{ P_n(f - g)^2 \wedge P(f - g)^2}.
	\end{align*}
	Finally, using the fact that
	\[
	\Ind[g(x) = 1/2 \text{ and } f(x) = y]
	\leq 4(f(x) - g(x))^2
	\]
	and
	\[
    	\Ind[g(x) = 1/2 \text{ and } f(x) \neq y]
    	\leq 4(f(x) - g(x))^2,
	\]
	we have, with probability at least $1 - \delta/4$, 
	\begin{align*}
	&
	\Bigg| P \Ind\left( g(X) = \frac12 \text{ and } f(X) = Y \right)
	- P_n \Ind\left( g(X) = \frac12 \text{ and } f(X) = Y \right) \Bigg|
	\\&
	\leq \gamma^2(n, \delta) + 2\gamma(n, \delta) \sqrt{ P_n(f - g)^2 \wedge P(f - g)^2}
	\end{align*}
	and
	\begin{align*}
	&
	\Bigg| P \Ind\left( g(x) = \frac12 \text{ and } f(x) \neq y \right) 
	- P_n \Ind\left( g(x) = 
	\frac12 \text{ and } f(x) \neq y \right) \Bigg|
	\\&
	\leq \gamma^2(n, \delta) + 2\gamma(n, \delta) \sqrt{ P_n(f - g)^2 \wedge P(f - g)^2}.
	\end{align*}
	Hence, by the union bound, with probability at least $1 - \delta$, we have
	\begin{align*}
	    &
	    \left| R^p(f) - R^p(g) - R^p_n(f) + R^p_n(g) \right|
	    \leq 4\gamma^2(n, \delta) + 8\gamma(n, \delta) \sqrt{ P_n(f 
		- g)^2 \wedge P(f - g)^2 }.
	\end{align*}
	The proof is complete.

\end{myproof}

\end{document}